\documentclass[journal]{IEEEtran}
\usepackage{amsmath}
\usepackage{amssymb}
\usepackage{cases}
\usepackage{graphicx}

\usepackage[sort,compress]{cite}
\usepackage{booktabs}

\usepackage{amsthm}

\newtheorem{theorem}{Theorem}

\theoremstyle{definition}

\newtheorem{remark}{{Remark}}

\usepackage{threeparttable}    
\usepackage{makecell}

\usepackage{arydshln} 
\usepackage{mdframed} 

\usepackage[]{xcolor}
\usepackage{subcaption}
\usepackage{wrapfig}
\usepackage{multirow}
\usepackage{mathtools}
\usepackage{empheq}
\usepackage{color}
\usepackage{bbm}
\usepackage{xspace}

\usepackage[shortlabels]{enumitem}

\usepackage{tikz} 
\usepackage{tkz-graph}
\usepackage{tkz-berge}
\usetikzlibrary{backgrounds,fit,shapes,snakes,arrows,shapes.geometric,positioning,automata}
\usetikzlibrary{intersections,patterns,shapes.misc}
\usetikzlibrary{decorations.pathmorphing}
\usepackage{pgfplots}

\tikzstyle{block} = [rectangle, rounded corners, minimum width=3cm, minimum height=1cm,text centered, draw=black, fill=red!30]
\tikzstyle{new} = [rectangle, rounded corners, minimum width=1cm, minimum
height=1cm,text centered, draw=black, fill=blue!10!white, dashed]
\tikzstyle{arrow} = [thick,->,>=stealth]
\usetikzlibrary{calc, quotes}

\usepackage{algorithmicx}
\usepackage{algorithm}
\usepackage[]{algpseudocode}

\usepackage{fontawesome}

\usepackage{soul}
\usepackage{comment}
\usepackage[]{url}
\usepackage[normalem]{ulem}
\useunder{\uline}{\ul}{}

\usepackage[pdftex, colorlinks=true, linkcolor=blue, citecolor = blue, urlcolor = cyan, filecolor=black, pagebackref=false, hypertexnames=false]{hyperref}
\usepackage{cleveref}
\crefname{problem}{Problem}{Problems}
\crefname{example}{Example}{Examples}
\crefname{section}{Sec.}{Secs.}
\Crefname{section}{Section}{Sections}
\Crefname{table}{Table}{Tables}
\crefname{table}{Table}{Tabs.}
\crefname{figure}{Fig.}{Figs.}
\crefname{algorithm}{Algorithm}{Algorithms}
\crefname{remark}{Remark}{Remarks}
\crefname{assumption}{Assumption}{Assumption}
\crefname{theorem}{Theorem}{Theorems}
\crefname{proposition}{Proposition}{Propositions}
\crefname{lemma}{Lemma}{Lemmas}
\crefname{corollary}{Corollary}{Corollaries}
\crefname{assumption}{Assumption}{Assumptions}
\crefname{definition}{Definition}{Definitions}

\usepackage{bbding}

\def\bfa{\mathbf{a}}
\def\bfb{\mathbf{b}}

\def\bfg{\mathbf{g}}

\def\bfn{\mathbf{n}}
\def\bfp{\mathbf{p}}

\def\bfv{\mathbf{v}}

\def\bfx{\mathbf{x}}
\def\bfy{\mathbf{y}}
\def\bfz{\mathbf{z}}

\def\bfA{\mathbf{A}}

\def\bfG{\mathbf{G}}
\def\bfH{\mathbf{H}}
\def\bfI{\mathbf{I}}

\def\bfM{\mathbf{M}}
\def\bfN{\mathbb{N}}

\def\bfP{\mathbf{P}}
\def\bfQ{\mathbf{Q}}
\def\bfR{\mathbf{R}}

\def\bfT{\mathbf{T}}

\def\bfX{\mathbf{X}}
\def\bfZo{\mathbf{0}}
\def\bfZ{\mathbf{Z}}
\def\bfLa{\mathbf{\Lambda}}
\def\bfpi{{\boldsymbol{\pi}}}

\def\bfNxyz{\mathbb{N}_{\text{p}}}
\def\bfNg{\mathbb{N}_{\theta}}

\def\bfomega{\boldsymbol{\omega}}

\def\bftheta{\boldsymbol{\theta}}

\def\bfPhi{\boldsymbol{\Phi}}

\newcommand{\smalloplus}{\mathbin{\text{\scalebox{1}{$\oplus$}}}}
\newcommand{\smallominus}{\mathbin{\text{\scalebox{1}{$\ominus$}}}}

\DeclareMathOperator*{\argmin}{arg\,min}

\newcommand{\bmat}{\begin{bmatrix}}
\newcommand{\emat}{\end{bmatrix}}

\providecommand{\optional}[1]{{}}

\providecommand{\techreport}[1]{{}}  

\begin{document}
\title{
Unobservable Subspace Evolution and Alignment for Consistent Visual–Inertial Navigation

}
\author{Chungeng Tian, Fenghua He, and Ning Hao
\thanks{Corresponding author: {Fenghua He}. The authors are with School of Astronautics, Harbin Institute of Technology, Harbin, 150000, China. {(email: tcghit@outlook.com; hefenghua@hit.edu.cn; haoning0082022@163.com)}}%
}



\maketitle
\begin{abstract}

The inconsistency issue in the Visual-Inertial Navigation System (VINS) is a long-standing and fundamental challenge. While existing studies primarily attribute the inconsistency to observability mismatch, these analyses are often based on simplified theoretical formulations that consider only prediction and SLAM correction. Such formulations fail to cover the non-standard estimation steps, such as MSCKF correction and delayed initialization, which are critical for practical VINS estimators. Furthermore, the lack of a comprehensive understanding of how inconsistency dynamically emerges across estimation steps has hindered the development of precise and efficient solutions. As a result, current approaches often face a trade-off between estimator accuracy, consistency, and implementation complexity. 
To address these limitations, this paper proposes a novel analysis framework termed Unobservable Subspace Evolution (USE), which systematically characterizes how the unobservable subspace evolves throughout the entire estimation pipeline by explicitly tracking changes in its evaluation points. This perspective sheds new light on how individual estimation steps contribute to inconsistency. Our analysis reveals that observability misalignment induced by certain steps is the antecedent of observability mismatch. Guided by this insight, we propose a simple yet effective solution paradigm, Unobservable Subspace Alignment (USA), which eliminates inconsistency by selectively intervening only in those estimation steps that induce misalignment. We design two USA methods: transformation-based and re-evaluation-based, both offering accurate and computationally lightweight solutions. Extensive simulations and real-world experiments demonstrate that the USE framework enables rigorous and intuitive consistency analysis. Furthermore, the proposed USA methods achieve superior or competitive performance in both accuracy and consistency compared to state-of-the-art approaches, while maintaining excellent computational efficiency and scalability.

\end{abstract}

 
\section{Introduction}

Visual-Inertial Navigation System (VINS), also referred to as Visual-Inertial Simultaneous Localization and Mapping (VI-SLAM), has become foundational for mobile robotics by fusing cameras and inertial measurement units to deliver accurate, robust, real-time state estimation without external infrastructure. In recent years, filter-based VINS estimators have gained increasing attention due to their favorable balance between accuracy and computational efficiency\cite{Davison2007Monoslam,mourikisMultiStateConstraintKalman2007,wu2015square,sun2018msckf, genevaOpenVINSResearchPlatform2020,fan2024SchurVINSSchurComplementbased,pengSqrt2025}. However, their performance is fundamentally limited by a long-standing challenge: inconsistency\cite{julier2001CounterExampleTheory}. In practice, inconsistency manifests as estimator overconfidence, most notably in the global yaw direction, where the reported uncertainty is significantly smaller than the true estimation error, ultimately degrading overall accuracy and robustness.
Existing studies have made significant advances in understanding this issue, attributing it to \textit{observability mismatch}: the dimension of the estimator's unobservable subspace does not match that of the true physical system \cite{huangAnalysisImprovementConsistency2008}.

Despite these advances, current studies on the inconsistency issue remain constrained by two critical limitations. 
First, a significant gap persists between the simplified theoretical models employed for consistency analysis and the VINS estimators used in real-world applications. The majority of prior work typically adopts a simplified Extended Kalman Filter (EKF)-SLAM formulation, comprising only the prediction and SLAM correction steps. This simplification is analytically convenient because these two steps directly correspond to the system's motion and measurement models, 
enabling direct use of the system's observability analysis for the estimator. However, practical VINS estimators rely heavily on non-standard steps, such as the Multi-State Constraint Kalman Filter (MSCKF) correction \cite{mourikisMultiStateConstraintKalman2007} and the delayed feature initialization \cite{li2014visual}. These steps break the direct correspondence to system models and fundamentally alter the structure of the estimator's unobservable subspace. 
 As a result, existing analysis tools have limitations in accurately evaluating these steps' impact on estimator observability and consistency.
Second, although it is widely acknowledged that observability mismatch introduces spurious information along unobservable directions, the field still lacks a mechanistic, step-wise understanding of how this erroneous information is generated within the estimation process. Specifically, it remains unclear how mismatches evolve dynamically across estimation steps, which hinders the design of targeted interventions for resolving inconsistency.

To address these gaps, we introduce Unobservable Subspace Evolution (USE), a novel analysis framework that shifts the focus of inconsistency mechanisms from static, system-level characterizations of estimator observability properties to the dynamic, step-wise mechanisms by which inconsistency forms. 
Specifically, USE goes beyond simply checking whether an estimator's unobservable subspace has the correct dimension. Instead, it investigates how individual estimation steps, including MSCKF correction and delayed feature initialization, impact the subspace, and especially, how they alter the evaluation point of the subspace. 
This perspective reveals a previously unrecognized antecedent of observability mismatch: \textit{observability misalignment}, where the estimator's unobservable subspace maintains the correct dimension but is not evaluated at the current estimate (e.g., at the previous state estimate). We show that specific estimation steps induce observability misalignment, which then evolves into an observability mismatch, thereby explaining how spurious information is dynamically injected.
Crucially, while previous studies have considered the influence of the Jacobian's evaluation points on dimensionality, USE for the first time reveals that the evaluation point of the unobservable subspace itself is the decisive factor governing the emergence of inconsistency.
By explicitly tracing how estimation steps perturb the evaluation point of the subspace, USE provides a principled mechanism-driven methodology for assessing the consistency of complex VINS estimators, including MSCKF and other sliding window filters (SWF).  In summary, USE
\begin{itemize}
    \item Shifts the focus of inconsistency mechanism analysis from the observability properties of the entire estimator to the influence of individual estimation steps on consistency.
    \item Identifies observability misalignment as the antecedent of observability mismatch and
    highlights the decisive role of evaluation points of unobservable subspace on consistency.
    \item Enables accurate consistency assessment for practical VINS estimators with non-standard steps,  thereby overcoming the limitations of existing analysis tools.
\end{itemize}

Building upon insights from USE, we further propose Unobservable Subspace Alignment (USA), a simple yet effective paradigm for resolving inconsistency by eliminating the observability misalignment at its source. 
Unlike existing solutions that attempt to constrain or redesign the entire estimator, USA precisely targets and corrects only the specific steps that induce misalignment. 
The main contributions of USA include:
 \begin{itemize}
    \item A targeted solution paradigm is proposed, which intervenes only in the steps inducing misalignment, representing the first consistency-improvement work that preserves Jacobian optimality without requiring state-independent unobservable subspaces.
    \item Two USA methods are presented: a transformation-based method that offers a straightforward closed-form solution, and an evaluation-based method that can be combined with transformation to further enhance accuracy without additional computational cost.
    \item For scenarios where multiple estimation substeps share a common linearization point, we develop two tailored strategies: one applicable to all cases, and the other specifically optimized for delayed feature initialization.
 \end{itemize}
 
 Finally, extensive simulations and real-world experiments are conducted to demonstrate the effectiveness of the proposed USE and USA. The results show that the USE framework enables rigorous and intuitive consistency analysis of practical VINS estimators. Moreover, USA achieves competitive performance (via transformation) or superior performance (via re-evaluation) in both consistency and accuracy compared to existing state-of-the-art methods, while maintaining computational efficiency and scalability.

\section{Related Work}

This section reviews research related to the inconsistency issue in VINS. For a broader overview of VINS, see \cite{durrant-whyte2006simultaneous,bailey2006simultaneous, cadena2016present,huangVisualInertialNavigationConcise2019,ebadi2024PresentFutureSLAM}.
\subsection{Mechanism of Inconsistency}

The inconsistency problem in VINS was first observed experimentally in \cite{julier2001CounterExampleTheory}, which showed that when a static robot repeatedly fuses relative position measurements from a landmark that has never been observed, the reported pose covariance, particularly the yaw angle, decreases in an anomalous and unreasonable manner. Follow-up studies further demonstrated experimentally that such inconsistency disappears when the estimator is linearized at the true state \cite{bailey2006ConsistencyEKFSLAMAlgorithm}, suggesting a connection between inconsistency and the choice of linearization points. Building on this insight, subsequent studies \cite{castellanos2004LimitsConsistencyEKFbased,shoudonghuangConvergenceConsistencyAnalysis2007,rodriguez-losada2006ConsistencyImprovementSLAM} attributed inconsistency to discrepancies between the linearization point and the true state.
While these early works provided valuable observations and the essential foundation for analyzing VINS inconsistency, their primary reliance on empirical observation means that a formal mathematical link between inconsistency and the linearization point has yet to be established. Formalizing this relationship is critical for improving the practical applicability of inconsistency solutions.

A key breakthrough came from \cite{huangAnalysisImprovementConsistency2008}, which examined the observability matrix evaluated at state estimates and identified that inconsistency arises from the observability mismatch. Specifically, as state estimates are updated, the linearization point of the estimator changes. This causes the inherently unobservable yaw angle to be erroneously treated as observable, which leads to inconsistency. 
The observability matrix-based analysis framework has been widely adopted in subsequent research, not only in VINS\cite{huangObservabilitybasedRulesDesigning2010,li2013HighprecisionConsistentEKFbased,heschConsistencyAnalysisImprovement2014} but also in cooperative localization \cite{huang2011ObservabilitybasedConsistentEKF,hao2022KDEKFKalmanDecomposition,hao2025consistency}, and map-based SLAM\cite{zhang2023ConsistentEfficientMapBased,gao2025NightvoyagerConsistentEfficient}. 
More recent studies  \cite{hao2025transformationbased,song2024affine} further proved that if the system's unobservable subspace remains independent of the state, changes in the linearization points do not alter the observability properties. As a result, inconsistency can be avoided by making the unobservable subspace constant.

In addition to the observability matrix-based analysis, the information matrix (or covariance matrix) has also been employed to study the inconsistency issue \cite{huang2011ObservabilityconstrainedSlidingWindow,chen2023OptimizationbasedVINSConsistency,barrauEKFSLAMAlgorithmConsistency2016,zhangConvergenceConsistencyAnalysis2017}. Huang et al. \cite{huang2011ObservabilityconstrainedSlidingWindow} computed the information matrix by substituting the Jacobian evaluated at the estimates and found that the dimension of the unobservable subspace decreased, which corroborated the results based on the observability matrix.

While prior research has firmly established the link between observability mismatch and inconsistency, it has overlooked a critical factor: the evaluation point of the unobservable subspace. This gap has left the dynamic mechanism, the precise step-by-step process by which inconsistency emerges, largely unexplored. To address this, our work introduces the USE, a novel approach for consistency analysis. USE provides a crucial refinement by considering not only the dimensionality but also the evaluation point of the unobservable subspace. This enhanced perspective allows us to trace how individual estimation steps contribute to the problem, ultimately revealing that observability misalignment is the antecedent of mismatch. Furthermore, we apply this analysis to the MSCKF estimator and demonstrate that its consistency property is not equivalent to that of the EKF-SLAM estimator, thereby correcting a prevalent, long-standing misconception in the literature \cite{li2013HighprecisionConsistentEKFbased,heschConsistencyAnalysisImprovement2014,hesch2014cameraimubased,huai2022robocentric}.

\subsection{Methods for Addressing Inconsistency}

Two main categories of solutions have been developed for the inconsistency problem:

\subsubsection{Observability constraints} These methods directly modify Jacobian matrices to ensure that the estimator has the correct dimension of the unobservable subspace.
After \cite{huangAnalysisImprovementConsistency2008} identified observability mismatch as the direct cause of inconsistency, the First-Estimates Jacobian (FEJ) method was promptly introduced. FEJ requires all Jacobian matrices to be linearized at the first estimated state, which helps prevent changes in the dimension of the unobservable subspace caused by variations in the linearization points. The FEJ method is conceptually simple and computationally efficient, but its main drawback is the lack of Jacobian optimality: when the initial estimates deviate significantly from the true values, large linearization errors can occur, leading to degraded performance or even divergence. Subsequent methods, such as Observability Constrained EKF (OC-EKF)\cite{heschConsistencyAnalysisImprovement2014,hesch2014cameraimubased} and FEJ2 \cite{chenFEJ2ConsistentVisualInertial2022} have been proposed to address this issue. 
OC-EKF ensures correct observability by optimizing the Jacobians rather than directly modifying the linearization point, while FEJ2 mitigates the linearization error introduced by FEJ. These methods achieve significant performance gains compared to FEJ, though the Jacobian optimality is not yet guaranteed.

\subsubsection{State-independent unobservable subspace}
\label{sec:method2}
This class of methods achieves state-independence of the unobservable subspace by reformulating the system state or transforming the error-state, thereby ensuring invariance of the unobservable subspace to changes in the linearization points \cite{hao2025transformationbased,song2024affine}. Representative approaches include Robocentric formulation, Right Invariant Extended Kalman Filter (RI-EKF), Equivariant Filters (EqF), and Transformed Extended Kalman Filter (T-EKF).

\textbf{Robocentric} formulation was first presented in \cite{castellanos2004LimitsConsistencyEKFbased}. It defines both the robot pose and the landmark states in the local coordinate frame of the robot body.
Later works generalized this concept by proposing camera-centric modeling \cite{civera20091point} and by integrating velocity along with spatiotemporal parameters \cite{bloesch2017IteratedExtendedKalman}. The Robocentric formulation has also been implemented in SWF \cite{huai2022robocentric} and square-root estimators \cite{huai2022SquarerootRobocentricVisualinertial,huai2024ConsistentParallelEstimation}.  Although the Robocentric formulation initially demonstrated better consistency, its underlying mechanism was not well understood. \cite{huai2022robocentric} pointed out that, under the Robocentric formulation, the system's unobservable subspace is independent of states, thus preventing inconsistency issues.
However, the design of the Robocentric formulation often relies on physical intuition. At present, there are no general guidelines for selecting Robocentric variables to ensure a state-independent unobservable subspace, especially for complex systems.

\textbf{RI-EKF}, as well as the more general I-EKF\cite{barrauInvariantExtendedKalman2017,barrau2018invariant,barrauGeometryNavigationProblems2023}, has attracted significant attention and application in robotic state estimation in recent years. For systems that possess both group-affine dynamics and output invariance\cite{hartley2020contactaided, yoon2024InvariantSmootherLegged}, the I-EKF can effectively reduce linearization errors and enhance convergence, making it the preferred choice. Due to the group-affine characteristic of IMU dynamics, modeling the IMU state using an invariant representation can yield a “banana-shaped” uncertainty distribution, which accurately captures the actual error distribution\cite{long2013banana,barfoot2014associating,brossard2022AssociatingUncertaintyExtended}. Although visual measurements do not exhibit output invariance, right-invariant parameterization ensures that the unobservable subspace remains independent of states, a property that has facilitated the adoption of RI-EKF in VINS. For instance, \cite{barrauEKFSLAMAlgorithmConsistency2016,zhangConvergenceConsistencyAnalysis2017} jointly modeled IMU states and feature points on the $\mathbf{SE}_{2+m}(3)$ group, while \cite{wuInvariantEKFVINSAlgorithm2017} subsequently integrated right-invariant state representation into the MSCKF framework. Nevertheless, identifying an appropriate invariance group for arbitrary systems can be nontrivial, and RI-EKF may involve additional computational overhead \cite{yangDecoupledRightInvariant2022,chen2024visualinertial}.

\textbf{EqF} offers a broader perspective for addressing system symmetries  \cite{vangoorEquivariantFilterEqF2023}. Unlike I-EKF, which focuses on directly modeling systems on matrix Lie groups, the core of EqF lies in identifying the equivariant symmetries inherent in the system dynamics and measurement equations \cite{fornasier2024EquivariantSymmetriesAided}. This framework is more flexible and can handle cases where the system dynamics do not strictly satisfy the group-affine property. When multiple symmetries are present (e.g., dealing with IMU bias \cite{fornasierEquivariantSymmetriesInertial2023,fornasier2022EquivariantFilterDesign,fornasier2022OvercomingBiasEquivariant,delama2024EquivariantIMUPreintegration}), EqF allows for selecting the most appropriate symmetry group based on specific task requirements. Additionally, when the system exhibits output equivariance, EqF can effectively eliminate second-order linearization errors in the measurement equations. Fortunately, visual measurements possess this equivariance \cite{vangoor2023eqvio}. Notably, EqF reduces to I-EKF when the system kinematics are group-affine \cite{vangoorEquivariantFilterEqF2023}.

\textbf{T-EKF} is a consistency-improved method proposed very recently. Its core idea is to apply a well-designed transformation to the error-state and perform state estimation based on the transformed error-state. T-EKF was first introduced in \cite{hao2022KDEKFKalmanDecomposition}, where the linearized error-state system is transformed into observable and unobservable components. Subsequent studies \cite{hao2025transformationbased} have further clarified the generality of this method: as long as the designed transformation ensures that the unobservable subspace of the transformed system is independent of the state, the consistency of the estimator can be guaranteed. Compared to methods such as Robocentric, RI-EKF, and EqF, a key advantage of T-EKF is that it provides a more direct and mathematically guided framework for ensuring consistency\cite{hao2025consistency}. However, T-EKF may sometimes encounter computational bottlenecks in covariance propagation similar to those of RI-EKF \cite{tian2025teskf}.

 Unlike existing methods, the proposed USA addresses the inconsistency issue at the level of individual steps rather than the entire estimator. Moreover, compared to the first category of methods, USA does not sacrifice the Jacobian optimality to enforce observability constraints, thereby achieving higher accuracy. In contrast to the second category of methods, USA does not need to reformulate the state or transform the error-state, nor does it require the unobservable subspace to be state-independent, while still offering superior or at least comparable performance.

\section{Problem Formulation}
\label{sec:system_model}
This section provides a brief description of the visual-inertial navigation system, its observability properties, and the issue of
 inconsistency.

\subsection{Visual-Inertial Navigation System} 
The system model employed here is equivalent to that used in OpenVINS \cite{genevaOpenVINSResearchPlatform2020}, with orientations represented by rotation matrices instead of JPL quaternions \cite{trawny2005indirect} to enhance readability. It is noteworthy that the proposed approaches are not limited to the specific model discussed here but are applicable to a wide range of VINS formulations.

\label{sec:system_state}

\textbf{System State:}
The system state in VINS includes the following variables: 
\begin{subequations}
    \begin{align}
    \bfx &= \begin{pmatrix}
    \bfx_{\mathcal{I}}, & \bfx_{\mathcal{W}}, & \bfx_{\mathcal{F}}
    \end{pmatrix}, \\
    \bfx_{\mathcal{I}} &=
    \setlength{\arraycolsep}{2.5pt}
    \begin{pmatrix}
    {^G}\bfR_{I_\tau},& {^G}\bfp_{I_\tau},& {^G}\bfv_{I_\tau},&\bfb_{g},& \bfb_{a}
    \end{pmatrix}, \\
    \bfx_{\mathcal{W}} & = \begin{pmatrix}
    \bfpi_{1},&\cdots ,& \bfpi_{n}\end{pmatrix}, \\
    \bfx_{\mathcal{F}} & = \begin{pmatrix}
    {^G} \bfp_{f_1},&\cdots, &{^G} \bfp_{f_m}
    \end{pmatrix},
    \end{align}
    \label{equ:system_state}%
    \end{subequations}
    where ${^G}\bfR_{I_\tau}\in \mathbf{SO}(3)$, ${^G}\bfp_{I_\tau} \in \mathbb{R}^{3}$, and ${^G}\bfv_{I_\tau} \in \mathbb{R}^{3}$ represent the current orientation, position, and velocity of the IMU in the global frame, respectively. 
    The terms $\bfb_{g} \in \mathbb{R}^{3} $ and $\bfb_{a} \in \mathbb{R}^{3} $ denote the biases of the gyroscope and accelerometer, respectively. $\bfpi_i$ is defined as \begin{equation}
    \bfpi_i = \begin{pmatrix}
    {^G}\bfR_{I_i}, {^G}\bfp_{I_i}
    \end{pmatrix}, \quad i \in \{1,\ldots,n\},
    \end{equation} indicating the pose of a cloned IMU at a previous time $t_i$.
    Meanwhile, the global position of the $j$-th feature is denoted by ${^G}\bfp_{f_j} \in \mathbb{R}^3$, for $j=1,\ldots,m$.
    Let $\hat \bfx$ denote the estimated system state. The error-state $\tilde \bfx \in \mathbb{R}^{N} \,(N=15+6n+3m)$ describes the difference between the true and estimated states as follows: \begin{equation}
    \bfx = \hat \bfx \boxplus \tilde \bfx.
    \label{eq:err_state}
    \end{equation} For IMU positions, velocities, biases, and feature positions, the operator $\boxplus$ is defined within the vector space, whereas for IMU orientations,  the operation is defined as \begin{equation}
    \bfR = \hat \bfR \, \text{Exp}(\tilde{\bftheta}),
    \end{equation} where $\tilde{\bftheta} \in \mathbb{R}^3$ is the error-state for orientations and $\text{Exp}(\cdot)$ represents the exponential map defined on $\mathbf{SO}(3)$.

    \textbf{IMU Dynamics:} The dynamic model for the current IMU state $\bfx_{\mathcal{I}}$ is expressed as: \begin{equation}
        \left\{
        \begin{array}{l}
        {^G}\dot\bfR_{I_\tau} = {^G}\bfR_{I_\tau} [\bfomega]_\times, \\
        {^G}\dot\bfp_{I_\tau} = {^G}\bfv_{I_\tau}, \\
        {^G}\dot\bfv_{I_\tau} = {^G}\bfR_{I_\tau} \bfa + \bfg + \bfn_a,\\
        \dot{\bfb}_g = \bfn_{wg}, \\
        \dot\bfb_a = \bfn_{wa},
        \end{array}
        \right.
        \label{equ:IMU_dynamics}
        \end{equation} where \begin{subequations}
        \begin{align}
        \bfomega & = \bfomega_m - \bfb_{g} - \bfn_{g}, \\
        \bfa & = {^G}\bfR_{I_\tau}(\bfa_m - \bfb_{a} - \bfn_{a}) + \bfg,
        \end{align}
        \end{subequations} Here, $\bfomega_m$ and $\bfa_m$ denote the raw measurements from the gyroscope and accelerometer, respectively. The terms $\bfn_{wg}$, $\bfn_{wa}$, $\bfn_{g}$, and $\bfn_{a}$ represent Gaussian white noise, while $\bfg$ denotes the gravity vector in the global frame.
        
        Within the dynamic equations, the state estimate can be integrated from $\hat{\bfx}_{k-1}$ to $\hat{\bfx}_{k}$ by setting the noise terms to zero. The corresponding error-state is propagated as follows: \begin{equation}
        \tilde{\bfx}_{k} = \bfPhi_{k-1} \tilde{\bfx}_{k-1} + \bfG_{k-1} \bfn_{k-1},\label{equ:less-propagation}
        \end{equation}
         where $\bfn = \begin{bmatrix} \bfn_{wg}^\top& \bfn_{wa}^\top& \bfn_{g}^\top& \bfn_{a}^\top \end{bmatrix}^\top$, \begin{align}
        \setlength{\arraycolsep}{0.1cm}
        \bfPhi_{k-1} &= \begin{bmatrix}
        \bfPhi_{\mathcal{I}} &\bfZo \\
        \bfZo & \bfI_{3m+6n}
        \end{bmatrix},
        \quad \bfG_{k-1} = \begin{bmatrix}
        \bfG_{\mathcal{I}} \\
        \bfZo_{(3m+6n)\times 12} \\
        \end{bmatrix},
        \end{align} $\bfPhi_\mathcal{I} \in \mathbb{R}^{15\times 15}$ and $\bfG_\mathcal{I} \in \mathbb{R}^{15\times 12}$ represent the error-state transition matrix and the noise propagation matrix for the IMU state, respectively. Given that the top-left $9 \times 9$ submatrix of $\bfPhi_\mathcal{I}$ is crucial for subsequent derivations, we explicitly present it as follows: \begin{equation}
        \bfPhi_{\mathcal{I}_{[1:9,1:9]}} = \begin{bmatrix}
        {^G}\hat{\bfR}_{I_{k}}^\top {^G}\hat\bfR_{I_{k-1}} & \bfZo & \bfZo_{3\times 3}\\
        -{^G}\hat{\bfR}_{I_{k-1}}[\Delta \hat\bfp_{k-1}]_\times & \bfI_3 & \bfI_3\Delta t \\
        -{^G}\hat{\bfR}_{I_{k-1}}[\Delta \hat\bfv_{k-1}]_\times & \bfZo & \bfI_3 \\
        \end{bmatrix},
        \label{equ:Phi19}
        \end{equation} where \begin{align}
        \Delta\hat{\bfp}_{k-1} &={^G}\hat{\bfR}_{I_{k-1}}^\top( {^G}\hat\bfp_{I_k}-{^G}\hat\bfp_{I_{k-1}}- {^G}\hat\bfv_{I_k}+\frac{1}{2}\bfg\Delta t^2),\\
        \Delta\hat{\bfv}_{k-1} &={^G}\hat{\bfR}_{I_{k-1}}^\top( {^G}\hat\bfv_{I_k}-{^G}\hat\bfv_{I_{k-1}}+ \bfg{\Delta t}).
        \end{align} For details regarding the remaining blocks of $\bfPhi_\mathcal{I}$ and $\bfG_\mathcal{I}$, the interested reader is referred to \cite{yangAnalyticCombinedIMU2020}.
        
        \textbf{Visual Measurements:} When the camera observes the $j$-th feature at time $t_i$, the visual measurement can be expressed as: \begin{equation}
            \bfy_{ij} = \begin{bmatrix}
            \frac{x}{z} & \frac{y}{z}
            \end{bmatrix}^\top + \boldsymbol{\epsilon}_{ij}, \quad ^{C_i}\bfp_{f_j} = \left[ \begin{matrix}
            x & y & z
            \end{matrix}\right]^\top,
            \label{equ:visual_measure}
            \end{equation} where \begin{equation}
            ^{C_i}\bfp_{f_j} = {^C}\mathbf{R}_{I} {^G}\mathbf{R}^\top_{I_i} ({^G}\mathbf{p}_{f_j}-{^G}\mathbf{p}_{I_i}) + {^C}\mathbf{p}_{I},
            \label{eq:GpI}
            \end{equation} and $\boldsymbol{\epsilon}_{ij} \in \mathbb{R}^2$ denotes the measurement noise, while $({^C}\bfR_I,{^C}\bfp_I)$ correspond to the camera's extrinsic parameters.
            
            Linearizing $\bfy_{ij}$ yields the following linearized measurement model: \begin{equation}
            \tilde{\bfy}_{ij} = \bfH_{ij}\tilde{\bfx} + \boldsymbol{\epsilon}_{ij},\label{equ:less-measurement}
            \end{equation} where the Jacobian is defined as: \begin{equation}
            \setlength{\arraycolsep}{0.1cm}
            \bfH_{ij} = \frac{\partial \bfy_{ij}}{\partial ^{C_i}\bfp_{f_j}}\begin
            {bmatrix}
            \cdots & \frac{\partial ^{C_i}\bfp_{f_j}}{\partial {^G}\tilde{\bftheta}_{I_i}} & \frac{\partial ^{C_i}\bfp_{f_j}}{\partial {^G}\tilde{\bfp}_{I_i}} & \cdots & \frac{\partial ^{C_i}\bfp_{f_j}}{\partial {^G}\tilde{\bfp}_{f_j}} & \cdots 
            \end{bmatrix},
            \label{equ:15}
            \end{equation}
            \begin{align}
                \frac{\partial ^{C_i}\bfp_{f_j}}{\partial {^G}\tilde{\bftheta}_{I_i}} &= {^C}\mathbf{R}_{I} \big[{^G}\mathbf{R}^\top_{I_i} ({^G}\mathbf{p}_{f_j}-{^G}\mathbf{p}_{I_i} )\big]_\times, \\
            \frac{\partial ^{C_i}\bfp_{f_j}}{\partial {^G}\tilde{\bfp}_{f_j}} &= {^C}\mathbf{R}_{I} {^G}\mathbf{R}^\top_{I_i}, \quad
            \frac{\partial ^{C_i}\bfp_{f_j}}{\partial {^G}\tilde{\bfp}_{I_i}} = -{^C}\mathbf{R}_{I} {^G}\mathbf{R}^\top_{I_i}.
            \label{equ:partial}
            \end{align} It is important to note that the Jacobians $\bfPhi_{k-1}$ and $\bfH_{ij}$ depend on state estimates. In VINS estimators, they are typically evaluated at the current best estimate to minimize linearization error, which is a standard practice in nonlinear state estimation.

            \subsection{Observability Properties} 
            System observability refers to the capacity of a system to recover its initial states from all available measurements and is an inherent property of the system \cite{hermann1977nonlinear}. 
            The system unobservable subspace is the tangent space to the unobservable manifold, which is the set of states indistinguishable from a given point $\bfx$ by measurements. For the VINS model described in this section, its unobservable subspace has been widely studied and is spanned by the columns of $\mathbb{N}({\bfx})$: 
             \begin{equation}
            \setlength{\arraycolsep}{3pt}
            \mathbb{N}({\bfx}) = \begin{bmatrix}
            \bfNxyz & \bfNg({\bfx})
            \end{bmatrix} = \left[\begin{array}{cc}
            \bfZo_{3\times 3}& -{^G}{\bfR}_{I_\tau}^\top \bfg \\
            \bfI_3 & [{^G}{\bfp}_{I_\tau}]_\times \bfg\\
            \bfZo_{3\times 3} & [{^G}{\bfv}_{I_\tau}]_\times \bfg\\
            \bfZo_{3\times 3} & \bfZo_{3\times 1}\\
            \bfZo_{3\times 3} & \bfZo_{3\times 1}\\
            \hdashline
            \vdots & \vdots \\
            \bfZo_{3\times 3} & -{^G}{\bfR}_{I_i}^\top \bfg \\
            \bfI_3 & [{^G}{\bfp}_{I_i}]_\times \bfg\\
            \vdots & \vdots \\
            \hdashline
            \vdots & \vdots \\
            \bfI_3 & [{^G}{\bfp}_{f_{j}}]_\times \bfg\\
            \vdots & \vdots \\
            \end{array}\right],
            \label{equ:unobservable_subspace}
            \end{equation} where $\bfNxyz \in \mathbb{R}^{N\times 3}$ represents global position, while $\bfNg({\bfx}) \in \mathbb{R}^N$, depending on the state, corresponds to global yaw.
            The three blocks separated by dashed lines correspond to the current IMU, cloned IMUs, and features, respectively.
            For brevity, we refer to $\mathbb{N}({\bfx})$ simply as the unobservable subspace throughout this paper.

            The estimator's unobservable subspace differs from that of the system. The former pertains to the states that cannot be distinguished based on the information actually utilized by the estimator, whereas the latter assumes full utilization of all measurements. 
            The general expression of the estimator's unobservable subspace cannot be directly given because it depends on the specific estimator. 
            A powerful tool for its study is the estimator's information matrix, which is calculated from the Jacobians used by the estimator and accurately reflects the perceived uncertainty of its estimate.
            In particular, according to
            \cite{huang2011ObservabilityconstrainedSlidingWindow}, the estimator's unobservable subspace $\mathbf{N}$ is equal to the null space of the estimator's information matrix: 
            \begin{equation}
                \mathbf{N} = \texttt{null}(\bfLa).
               \label{equ:La_N}
            \end{equation}
            Note that \eqref{equ:La_N} requires the estimator to start without prior knowledge in the unobservable directions, since our interest lies in how spurious information is acquired along these directions.

            \subsection{Inconsistency of Standard Estimator}
            The filter-based VINS estimator utilizing the aforementioned VINS model, commonly referred to as the standard estimator, is known to suffer from inconsistency. Prior research \cite{huangAnalysisImprovementConsistency2008} has primarily attributed this issue to a dimensional mismatch between the unobservable subspace of the estimator and that of the underlying system:
            \begin{equation}
                \texttt{dim}({\textbf{N}})\neq \texttt{dim}(\bfN(\bfx)) = 4.
            \end{equation}
            In the following sections, we present a comprehensive analysis that considers not only the dimensionality but also the evaluation point of the estimator's unobservable subspace. Based on this broader perspective, we then present a simple yet effective paradigm for resolving the inconsistency.

\section{Unobservable Subspace Evolution} 

This section presents a novel analysis framework, named Unobservable Subspace Evolution (USE). In contrast to prior methods that rely on static observability matrices to assess consistency, USE uncovers the dynamic mechanisms underlying the formation of inconsistency. Specifically, to isolate and quantify the contribution of each estimation step to inconsistency, we first classify the unobservable subspaces of the estimator into three distinct statuses, and then regard estimation steps as actions that induce status evolution. Furthermore, we theoretically establish that the status evolutions form a Markov process. Most importantly, our analysis identifies observability misalignment as the antecedent of observability mismatch and pinpoints the specific steps that lead to it. This analysis of the deeper cause, along with the precise identification of critical steps, provides a theoretical foundation for the design of consistent estimators in the following section.

As a starting point, the analysis in this section is restricted to prediction and correction\footnote{The correction discussed here is also known as the SLAM correction/update or EKF correction/update, to distinguish it from the MSCKF correction/update.}. Other steps, such as MSCKF correction and delayed feature initialization, also satisfy the Markov property and will be discussed in detail in Section \ref{sec:VINS_Estimator}.

\label{sec:use}

\subsection{Status Evolution of Unobservable Subspace}
\label{sec:status_transition}
It is challenging to quantify the contribution of each estimation step to inconsistency in a time-sequence analysis framework because the estimation steps are continuously interleaved and interact with one another. 
To isolate the effect of each step from the influence of prior history, we 
divide the estimator’s unobservable subspace into distinct statuses and regard each estimation step as an action that triggers status evolution \footnote{To avoid confusion with system state variables, we use the term “status evolution” instead of “state transition” throughout this paper.}. 

The estimator's unobservable subspaces \eqref{equ:La_N} are categorized into the following three statuses based on their dimensions and evaluation points:
\begin{equation}
    \text{Status} = \left\{
    \begin{array}{ll}
        \text{Mismatched}, \quad & \text{if } \texttt{dim}(\mathbf{N}) \neq 4 ; \\
        \text{Aligned}, \quad & \text{else if } \mathbf{N} = \mathbb{N}(\hat{\bfx}^{*}); \\
        \text{Misaligned}, \quad &\text{otherwise};
    \end{array}    
    \right.
\end{equation}
where $\hat{\bfx}^{*}$ is the current estimate generated by the corresponding estimation step. Note that Mismatched is a problematic status recognized as the cause of inconsistency in previous work\cite{huangAnalysisImprovementConsistency2008}, while Aligned and Misaligned are newly defined in this study, and both possess the correct dimension.
Compared to prior studies that focus only on the dimension of the estimator's unobservable subspace, our analysis additionally incorporates its evaluation point. The subsequent analysis demonstrates that the evaluation point critically influences consistency. 

\begin{figure}[htbp]
    \centering
    \resizebox{0.47\textwidth}{!}{\begin{tikzpicture}[->,>=stealth,node distance=5cm,auto,thick]
      
    \definecolor{lightgreen}{RGB}{200, 230, 201}
    \definecolor{lightyellow}{RGB}{255, 249, 196}
    \definecolor{lightred}{RGB}{255, 205, 210}
    \definecolor{darkgreen}{RGB}{56, 142, 60}
    \definecolor{darkorange}{RGB}{239, 108, 0}
    \definecolor{darkred}{RGB}{198, 40, 40}
    \definecolor{predictionblue}{RGB}{33, 150, 243}  
    \definecolor{correctionpurple}{RGB}{156, 39, 176}  
    \definecolor{brightred}{RGB}{255, 0, 0}  
    
    \node[state, 
          fill=lightgreen, 
          draw=darkgreen, 
          line width=2pt,
          minimum size=2.5cm,
          font=\bfseries,
          ] (A) {Aligned};
           
    \node[state,
          right of=A,
          shape=circle,
          fill=lightyellow,
          draw=darkorange,
          line width=2pt,
          minimum size=2.5cm,
          font=\bfseries,
          ] (M) {Misaligned};
          
    \node[state,
          right of=M,
          shape=circle,
          fill=lightred,
          draw=darkred,
          line width=2pt,
          minimum size=2.5cm,
          font=\bfseries,
          ] (MM) {Mismatched};
  
    \path (A) edge[loop above, 
                   line width=1.2pt,
                   color=predictionblue] 
                   node[above, color=predictionblue, font=\bfseries]{Prediction} (A);
                   
    \path (M) edge[loop above,
                   line width=1.2pt,
                   color=predictionblue] 
                   node[above, color=predictionblue, font=\bfseries]{Prediction} (M);
                   
    \path (MM) edge[loop above,
                    line width=1.2pt,
                    color=predictionblue] 
                    node[above, color=predictionblue, font=\bfseries]{Prediction} (MM);
  
    \path (MM) edge[loop below,
                    line width=1.2pt,
                    color=correctionpurple] 
                    node[below, color=correctionpurple, font=\bfseries]{Correction} (MM);
  
    \path (A) edge[
                   above,
                   line width=1.2pt,
                   color=correctionpurple] 
                   node[above, color=correctionpurple, font=\bfseries]{Correction} (M);
  
    \path (M) edge[
                   above,
                   line width=1.2pt,
                   color=correctionpurple] 
                   node[above, color=correctionpurple, font=\bfseries]{Correction} (MM);
  
  
  
  \end{tikzpicture}}
    \caption{Status evolutions triggered by prediction and correction in the standard estimator.}
    \label{fig:state_transition}
    \vspace{-0.5cm}
\end{figure}
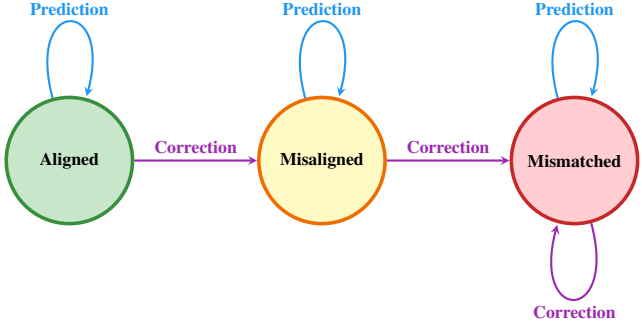

To ensure that the influence of each step on consistency is isolated, we theoretically show that the status evolutions satisfy the Markov property: the next status depends solely on the current status and the current estimation step, without reliance on historical steps, as illustrated in Fig.~\ref{fig:state_transition}.  
We summarize this property in Theorem~\ref{lamma:1}.

\begin{theorem}
The status evolution of the unobservable subspace exhibits the Markov property with respect to estimation steps, a relationship modeled by:
    \begin{subequations}
    \begin{align}
    & \delta(\text{Aligned},\text{Prediction})=\text{Aligned}, \label{equ:20a} \\
    & \delta(\text{Aligned},\text{Correction})=\text{Misaligned}, \label{equ:20b} \\
    & \delta(\text{Misaligned},\text{Correction})=\text{Mismatched}, \label{equ:20c} \\
    & \delta(\text{Misaligned},\text{Prediction})=\text{Misaligned}, \\
    & \delta(\text{Mismatched},\text{Prediction})=\text{Mismatched}, \\
    & \delta(\text{Mismatched},\text{Correction})=\text{Mismatched},
    \end{align}
    \end{subequations}    
    where $\delta(s,a)$ indicates that given the current status $s$ and estimation step $a$, the estimator evolves to a new status $\delta(s,a)$.
\label{lamma:1}
\end{theorem}

\begin{proof}
    To save space, we present the proof for only the first three equations; the remaining three can be proved similarly. 
 The proofs for these status evolutions follow the same analytical trajectory. Since the information matrix is orthogonal to the estimator's unobservable subspace, we first analyze how the information matrix evolves during estimation (see \eqref{equ:la_prediction}, \eqref{equ:La_update}, and \eqref{equ:31}). This analysis then helps infer the status evolution of the estimator's unobservable subspace (see \eqref{equ:N_prediction}, \eqref{equ:N_correction}, and \eqref{equ:N_correction2}).

    \textbf{Proof of \eqref{equ:20a}:}
    Denote   $(\hat{\bfx}_{k-1}^{\smalloplus},{\bfLa}_{k-1}^{\smalloplus})$ and $(\hat{\bfx}_k^{\smallominus},{\bfLa}_k^{\smallominus})$ as the estimates before and after the prediction, respectively.
    Since the unobservable subspace is Aligned at the beginning of the prediction, we have
    \begin{equation}
    \mathbf{N}_{k-1}^{\smalloplus} = \bfN(\hat\bfx _{k-1}^{\smalloplus} ).
         \label{equ:La_N_k_1}
    \end{equation}
    During prediction, the information is propagated as:
    \begin{subequations}
        \begin{align}
            \bfLa_{k}^{\smallominus} &= 
                \left(\bfQ_{k-1} + \bfPhi_{k-1} {\bfLa_{k-1}^{\smalloplus}}^{-1} \bfPhi_{k-1}^{\top}\right)^{-1}
            \\
            &= \underbrace{\bfQ^{-1}_{k-1} \big(\bfQ^{-1}_{k-1} + \bfPhi_{k-1}^{-\top} \bfLa_{k-1}^{\smalloplus} \bfPhi_{k-1}^{-1}\big)^{-1} \bfPhi_{k-1}^{-\top}}_{\bfZ} \bfLa_{k-1}^{\smalloplus} \bfPhi_{k-1}^{-1}, \\
            & = \bfZ  \bfLa_{k-1}^{\smalloplus} \bfPhi_{k-1}^{-1},
            \label{equ:La_N_k}
        \end{align}
        \label{equ:la_prediction}
    \end{subequations}
    where 
    \begin{equation} 
        \bfQ_{k-1} = \bfG_{k-1}\mathbb{E}(\bfn_{k-1}\bfn_{k-1}^\top) \bfG_{k-1}^\top.
    \end{equation}
    Combining  \eqref{equ:La_N}, \eqref{equ:La_N_k_1}, and  \eqref{equ:La_N_k}, we have:
    \begin{equation}
        \bfLa_{k}^{\smallominus} \bfPhi_{k-1}\bfN(\hat\bfx _{k-1}^{\smalloplus} ) = \bfZo.
        \label{equ:25}
    \end{equation}
    Then we can see that after the prediction step, the unobservable subspace is evaluated at the current estimate, i.e., Aligned: 
    \begin{subequations}
        \begin{align}
            \mathbf{N}_{k}^{\smallominus} &\overset{\eqref{equ:La_N}} {=} \texttt{null}(\bfLa_{k}^{\smallominus})\\
            & \overset{\eqref{equ:25}} {=} \bfPhi_{k-1}\bfN(\hat\bfx _{k-1}^{\smalloplus} ) \\
            & \overset{\eqref{equ:Phi19}}{=}\bfN(\hat\bfx _{k}^{\smallominus} ) \label{equ:La_N_prop}
        \end{align}
        \label{equ:N_prediction}
    \end{subequations}

    \textbf{Proof of \eqref{equ:20b}:}
    At the beginning of the correction, the unobservable subspace is Aligned, given by \eqref{equ:La_N_prop}.  
    During the correction, the estimator utilizes the measurements $\bfz_k$ to correct
    the estimate from $(\hat{\bfx}_k^{\smallominus},{\bfLa}_k^{\smallominus})$ to $(\hat{\bfx}_k^{\smalloplus},{\bfLa}_k^{\smalloplus})$ where 
    \begin{equation}
        \bfLa_{k}^{\smalloplus} = \bfLa_{k}^{\smallominus} + {\bfH_{k}}^\top \bfR_{k}^{-1} \bfH_{k},
        \label{equ:La_update}
    \end{equation}
     and $\bfz_k$, $\bfH_{k}$, and $\bfR_{k}$ are the row/row/diagonal stacks of $\{\dots, \bfz_{ij}, \dots\}$, $\{\dots, \bfH_{ij}, \dots\}$, and $\{\dots, \bfR_{ij}, \dots\}$, respectively.
    Since $
    \bfH_{k}$ is evaluated at $\hat{\bfx}_{k}^{\smallominus}$, we have:
    \begin{equation}
        \bfH_{k}\mathbb{N}(\hat{\bfx}_{k}^{\smallominus})=\bfZo.
        \label{equ:28}
    \end{equation}
    Then, we can find that the unobservable subspace is still evaluated at the previous estimate $\hat\bfx_{k}^{\smallominus}$:
    \begin{subequations}
        \begin{align}
                   \mathbf{N}_{k}^{\smalloplus} &\overset{\eqref{equ:La_N}} {=} \texttt{null}(\bfLa_{k}^{\smalloplus})\\
                   &\overset{\eqref{equ:La_update}}{=} \mathbb{N}(\hat{\bfx}_{k}^{\smallominus}) \cap  \texttt{null} ({\bfH_{k}}^\top \bfR_{k}^{-1} \bfH_{k}) \\
                   &\overset{\eqref{equ:28}} {=} \mathbb{N}(\hat{\bfx}_{k}^{\smallominus}) \label{equ:La_N_Acc}
        \end{align}
        \label{equ:N_correction}%
    \end{subequations}
    which indicates that after the correction step, the unobservable subspace still has four unobservable directions, but is not evaluated at the current estimate, i.e., Misaligned.
 
    \textbf{Proof of \eqref{equ:20c}:}
    Denote $(\hat{\bfx}_k^{\smalloplus},{\bfLa}_k^{\smalloplus})$ and $(\hat{\bfx}_k^{\smalloplus\smalloplus},{\bfLa}_k^{\smalloplus\smalloplus})$ as the estimates before and after the correction with the measurement $\bfz_k'$, respectively.
    Since the unobservable subspace is Misaligned at the beginning, we 
    have \begin{equation}
         \mathbf{N}_k^{\smalloplus} =  \bfN(\hat\bfx _{k}^\text{M} ), 
    \end{equation}
    where $\hat\bfx _{k}^\text{M} \neq \hat\bfx _{k}^{\smalloplus}$. 
    During the correction, the information matrix is corrected as 
\begin{equation}
    \bfLa_{k}^{\smalloplus\smalloplus} = \bfLa_{k}^{\smalloplus} + {\bfH_{k}'}^\top \bfR_{k}'^{-1} \bfH_{k}',
    \label{equ:31}
\end{equation}
where $\bfH_{k}'$ and $\bfR_{k}'$ are the Jacobian and noise matrix corresponding to $\bfz_k'$, respectively. Because $\bfH_{k}'$ is evaluated at $\hat{\bfx}_{k}^{\smalloplus}$ and $ \hat\bfx _{k}^{\smalloplus} \neq \hat\bfx _{k}^\text{M}$, there is (almost) zero probability in practice that ${\bfH_{k}'}  \bfNg(\hat\bfx_{k}^\text{M}) = \bfZo$ holds (see \eqref{equ:15} and \eqref{equ:unobservable_subspace}). Therefore, we have
\begin{equation}
    \bfNg(\hat\bfx_{k}^{\text{M}}) \cap \texttt{null}({\bfH_{k}'}^\top \bfR_{k}'^{-1} \bfH_{k}') = \{\bfZo\},
    \label{equ:32}
\end{equation}
and 
\begin{subequations}
    \begin{align}
        \mathbf{N}_k^{\smalloplus\smalloplus} & \overset{\eqref{equ:La_N}} {=} \texttt{null}(\bfLa_{k}^{\smalloplus\smalloplus}) \\
        & \overset{\eqref{equ:31}}{=} \bfN(\hat\bfx _{k}^\text{M} ) \cap \texttt{null}({\bfH_{k}'}^\top \bfR_{k}'^{-1} \bfH_{k}') \\
        & \overset{\eqref{equ:unobservable_subspace}}{=} \bfNxyz \cup \bfNg(\hat\bfx_{k}^{\text{M}}) \cap \texttt{null}({\bfH_{k}'}^\top \bfR_{k}'^{-1} \bfH_{k}') \\
        & \overset{\eqref{equ:32}}{=} \bfNxyz. \label{equ:collapse}
    \end{align}
    \label{equ:N_correction2}%
\end{subequations}
\eqref{equ:collapse} indicates that the unobservable subspace collapses to three dimensions after the correction, i.e., Mismatched. Specifically, $\bfNg$ is no longer in the null space of the information matrix, indicating that global yaw falsely becomes observable in the estimator.
\end{proof}

The above proof can be visualized using the quadratic cost function \(\mathbf{J}_{\hat{\mathbf{x}}}(\mathbf{x}) = (\mathbf{x} - \hat{\mathbf{x}})^\top \mathbf{\Lambda} (\mathbf{x} - \hat{\mathbf{x}})\), as shown in Fig. \ref{fig:illumiation}, where the system dimension is reduced to two with one unobservable direction for visualization purposes. If the state is unobservable, the inequality \(\mathbf{J}_{\hat{\mathbf{x}}}(\mathbf{x}) < c\) forms a cylinder extending infinitely along the unobservable direction, as illustrated in Figs. \ref{fig:illumiation}(b)-(d). When the state erroneously becomes observable, the cylinder degenerates into an ellipsoid, as shown in Fig. \ref{fig:illumiation}(e).

\begin{figure}[!t]
    \centering
    \includegraphics[width=0.485\textwidth]{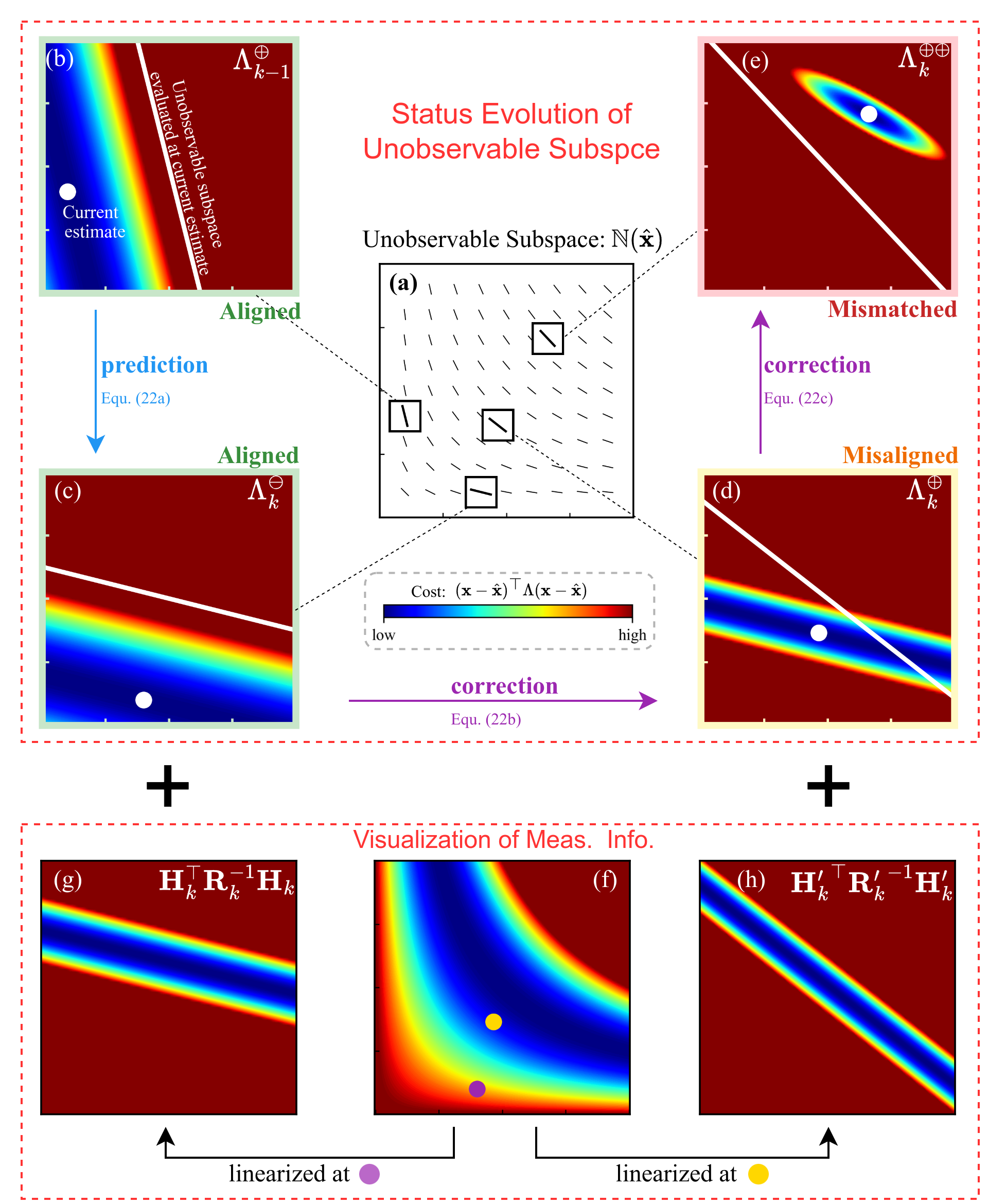}
    \caption{A visual illustration for proof of Theorem \ref{lamma:1}.
    (a) is the state-dependent unobservable subspace where the Aligned unobservable subspaces are evaluated at the current estimates. (b)-(c): After the prediction step, the unobservable subspace is evaluated at the current estimate, i.e., Aligned. 
    (c)-(d): However, after the correction step, the unobservable subspace is still evaluated at the previous estimate, i.e., Misaligned. (d)-(e): The observability misalignment eventually leads to a reduction of the unobservable dimension, i.e., Mismatched. 
    (f): Visualization of nonlinear measurement. (g) and (h): Linearized information evaluated at $\hat{\bfx}_{k}^{\smallominus}$ and $\hat{\bfx}_{k}^{\smalloplus}$, respectively.}
    \label{fig:illumiation}
\end{figure}
\subsection{Dynamic Mechanism of Inconsistency Formation}
Previous studies primarily attributed inconsistency in VINS to a dimensional mismatch within the unobservable subspace. Our analysis, based on the USE framework, uncovers a deeper dynamic mechanism: \textbf{observability misalignment precedes and induces mismatch.}
Theorem \ref{lamma:1} demonstrates that any estimator that eventually becomes Mismatched must first pass through an intermediate status of {Misaligned}. In other words, perturbations in the evaluation point of the subspace can diminish its dimension, thereby leading to inconsistency.
This insight advances prior, incomplete explanations by emphasizing the decisive role of the subspace evaluation point in shaping inconsistency. Moreover, it establishes the theoretical foundation for the subsequent USA paradigm.

However, not every estimation step necessarily induces misalignment (for instance, prediction does not). In practical VINS estimators, beyond prediction and correction, additional operations such as state augmentation, marginalization, and delayed feature initialization are also employed. This raises a natural question: do the estimation steps that cause misalignment exhibit common, identifiable characteristics that enable us to recognize them systematically?
The following theorem provides a criterion for determining whether a given step introduces observability misalignment.
\begin{theorem}
    
    An estimation step will cause the status to evolve from {Aligned} to {Misaligned} if the following conditions hold:
    \begin{itemize}
        \item It corrects the estimate from  $\small (\hat{\bfx}^{\smallominus}_k, \bfLa^{\smallominus}_k)$ to $\small (\hat{\bfx}^{\smalloplus}_k, \bfLa^{\smalloplus}_k)$;
        \item $\bfN(\hat{\bfx}^{\smallominus}_k) \neq \bfN(\hat{\bfx}^{\smalloplus}_k)$.
    \end{itemize}
    In this case, the resulting Misaligned unobservable subspace is evaluated at the previous estimate $\hat{\bfx}^{\smallominus}_k$.
    \label{the:2}
\end{theorem}
\begin{proof}
  Since both the Jacobians and the null space of $\bfLa^{\smallominus}_k$ are evaluated at the previous estimate $\hat{\bfx}^{\smallominus}_k$, the null space of $\bfLa^{\smalloplus}_k$ is likewise computed with respect to $\hat{\bfx}^{\smallominus}_k$. As a result, the estimator’s unobservable subspace continues to be evaluated at the previous estimate. Because the unobservable subspaces evaluated at the previous and current estimates are not equivalent, the unobservable subspace becomes {Misaligned}.
\end{proof}
It is worth noting that the prediction step does not satisfy the first condition, because the state estimates before and after prediction correspond to different timestamps. Similarly, state augmentation and marginalization also fail to meet the first condition, as they merely restructure the estimator’s state representation rather than performing a correction. For the estimators described in Section~\ref{sec:method2}, their unobservable subspaces are fully constant. Consequently, the second condition never holds, thereby ensuring consistency.

\section{Unobservable Subspace Alignment}
\label{sec:usa}

This section proposes a paradigm named Unobservable Subspace Alignment (USA) to address the inconsistency issue.
USA can be introduced as an additional step incorporated into existing estimators. Specifically, it is applied immediately when the status evolves to Misaligned, thereby preventing the status from further evolving to Mismatched, as shown in Fig. \ref{fig:state_transition2}.
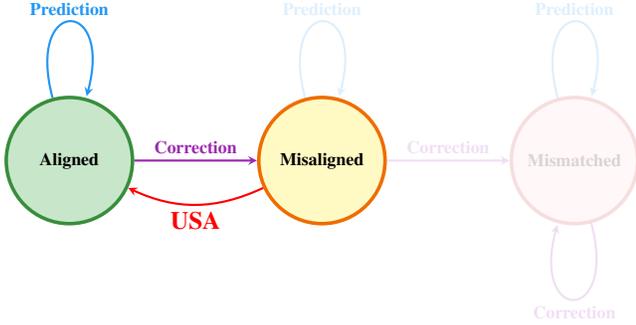
\begin{figure}[t]
    \centering
    \resizebox{0.47\textwidth}{!}{\begin{tikzpicture}[->,>=stealth,node distance=5cm,auto,thick]
      
    \definecolor{lightgreen}{RGB}{200, 230, 201}
    \definecolor{lightyellow}{RGB}{255, 249, 196}
    \definecolor{lightred}{RGB}{255, 205, 210}
    \definecolor{darkgreen}{RGB}{56, 142, 60}
    \definecolor{darkorange}{RGB}{239, 108, 0}
    \definecolor{darkred}{RGB}{198, 40, 40}
    \definecolor{predictionblue}{RGB}{33, 150, 243}  
    \definecolor{correctionpurple}{RGB}{156, 39, 176}  
    \definecolor{brightred}{RGB}{255, 0, 0}  
    
    \node[state, 
          fill=lightgreen, 
          draw=darkgreen, 
          line width=2pt,
          minimum size=2.5cm,
          font=\bfseries,
          ] (A) {Aligned};
          
    \node[state,
          right of=A,
          shape=circle,
          fill=lightyellow,
          draw=darkorange,
          line width=2pt,
          minimum size=2.5cm,
          font=\bfseries,
          ] (M) {Misaligned};
          
    \node[state,
          right of=M,
          opacity=0.15,
          shape=circle,
          fill=lightred,
          draw=darkred,
          line width=2pt,
          minimum size=2.5cm,
          font=\bfseries,
          ] (MM) {Mismatched};
  
    \path (A) edge[loop above, 
                   line width=1.2pt,
                   color=predictionblue] 
                   node[above, color=predictionblue, font=\bfseries]{Prediction} (A);
                   
    \path (M) edge[loop above,
                   line width=1.2pt,
                   opacity=0.15,
                   color=predictionblue] 
                   node[above, color=predictionblue, font=\bfseries]{Prediction} (M);
                   
    \path (MM) edge[loop above,
                    line width=1.2pt,
                   opacity=0.15,
                    color=predictionblue] 
                    node[above, color=predictionblue, font=\bfseries]{Prediction} (MM);
  
    \path (MM) edge[loop below,
                    line width=1.2pt,
                   opacity=0.15,
                    color=correctionpurple] 
                    node[below, color=correctionpurple, font=\bfseries]{Correction} (MM);
  
    \path (A) edge[
                   above,
                   line width=1.2pt,
                   color=correctionpurple] 
                   node[above, color=correctionpurple, font=\bfseries]{Correction} (M);
  
    \path (M) edge[
                   above,
                   opacity=0.15,
                   line width=1.2pt,
                   color=correctionpurple] 
                   node[above, color=correctionpurple, font=\bfseries]{Correction} (MM);
  
    \path (M) edge[bend left=25,
                   below,
                   line width=1.2pt,  
                   color=brightred] 
                   node[below, color=brightred, font=\Large\bfseries]{USA} (A);
  
  
  \end{tikzpicture}}
    \caption{Status evolutions when USA is integrated into the estimator. 
    Some evolutions originating from the Misaligned and Mismatched are shown as semi-transparent because they do not occur.}
    \label{fig:state_transition2}
\end{figure}
In this section, we present two methods to realize USA: transformation and re-evaluation.
The former is more general and can be applied to all estimators, while the latter is applicable to specific cases but is more accurate.

\subsection{Transformation-based USA}
Let the estimates before and after a step be $(\hat{\bfx}^{\smallominus}, \bfLa^{\smallominus})$ and $(\hat{\bfx}^{\smalloplus}_k, \bfLa^{\smalloplus}_k)$, respectively\footnote{Here, we omit the time subscript $k$ for brevity, as all the variables and matrices refer to the same timestamp.}. If this step induces misalignment, it can be eliminated using the following method:
\begin{theorem}
    The misalignment is eliminated solely by applying a nonsingular matrix $\mathbf{T} \in \mathbb{R}^{N \times N}$ to the information matrix: 
    \begin{align}
    \bfLa^{\smalloplus} \gets  \bfT^\top \bfLa^{\smalloplus} \bfT,\label{equ:La_align} 
\end{align} 
where \begin{equation}
    \bfT = \mathbb{N}(\hat\bfx^{\smallominus}) \mathbb{N}(\hat\bfx^{\smalloplus})^{\dagger}+\bfM\left(\bfI- \mathbb{N}(\hat\bfx^{\smalloplus})  \mathbb{N}(\hat\bfx^{\smalloplus})^{\dagger}\right),
    \label{equ:general_solution}
\end{equation}
$
    \bfN(\hat\bfx^{\smalloplus})^{\dagger} = \left(  \bfN(\hat\bfx^{\smalloplus})^\top  \bfN(\hat\bfx^{\smalloplus}) \right)^{-1} \bfN(\hat\bfx^{\smalloplus})^\top,
$ and  $\bfM \in \mathbb{R}^{N\times N}$ is an arbitrary matrix.
\label{the:3}
\end{theorem}
\begin{proof}
     Based on \eqref{equ:La_N}, one can see that the transformation to the information matrix \eqref{equ:La_align} correspondingly changes the estimator's unobservable subspace:
\begin{equation}
    \mathbf{N}^{\smalloplus} \gets \mathbf{T}^{-1}  \mathbf{N}^{\smalloplus}.
\end{equation}
According to Theorem \ref{the:2}, the Misaligned unobservable subspace is evaluated at the previous estimate, i.e., $\mathbf{N}^{\smalloplus} =\bfN(\hat{\bfx}^{\smallominus})$. Therefore, to align the unobservable subspace from the previous estimate $\hat{\mathbf{x}}^{\ominus}$ to the current estimate $\hat{\mathbf{x}}^{\oplus}$, the transformation matrix is designed as:
\begin{equation}
    \mathbb{N}(\hat{\bfx}^{\smalloplus}) =    \bfT^{-1}  \mathbb{N}(\hat{\bfx}^{\smallominus}) . \label{equ:N_alignment}
\end{equation}
Since \eqref{equ:N_alignment} takes the classical form $\bfA = \bfX^{-1}\mathbf{B}$ \cite{mitra1973common}, its general solution is directly given by \eqref{equ:general_solution}.
\end{proof}

Noting that $\bfM$ in \eqref{equ:general_solution} is an arbitrary matrix to be determined, we present two ways of finding the proper transformation matrices. We classify these two ways as {\it indirect transformation} and {\it direct transformation} based on whether they require making the unobservable subspace fully constant, as shown in Fig. \ref{fig:IndiandDI}.
\begin{figure}[!htbp]
    \centering 
    \includegraphics[width=0.44\textwidth]{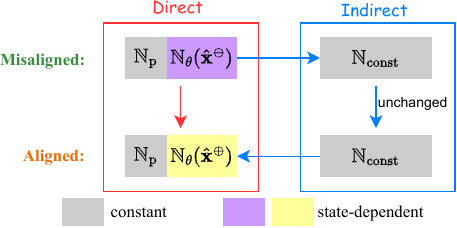}
    \caption{Unobservable subspace alignment through direct and indirect transformations. $\hat{\bfx}^{\smallominus}$ and $\hat{\bfx}^{\smalloplus}$ are the previous and the current estimates, respectively.
         The direct transformation matrix has a closed-form solution.  The indirect transformation additionally requires the design of an auxiliary matrix to make the unobservable subspace fully constant.}   
         \label{fig:IndiandDI}
\end{figure}

\textbf{Direct transformation:}
The transformation matrix will inevitably affect the uncertainty of the observable component.
Thus, the goal is to make the transformation matrix $\bfT$ as close to the identity matrix as possible to minimize its impact on the observable component:
\begin{equation}
    \bfM^*=  \argmin_{\bfM} \left\| \bfT - \bfI \right\|_F ^2,
    \label{equ:criteria}
\end{equation}
where $\left\| \cdot \right\|_F$ is the Frobenius norm.
By solving the optimization problem above, we can obtain a special solution to \eqref{equ:N_alignment}:
\begin{equation}
    \bfT = \bfI +  \left( \mathbb{N}(\hat\bfx^{\smallominus}) - \mathbb{N}(\hat\bfx^{\smalloplus}) \right)\mathbb{N}(\hat\bfx^{\smalloplus})^{\dagger},
    \label{equ:usa2}
\end{equation}
 with  
\begin{equation}
    \bfM^* = \left(\bfI - \mathbb{N}(\hat\bfx^{\smallominus}) \mathbb{N}(\hat\bfx^{\smalloplus})^{{\dagger}}\right)\left(\bfI- \mathbb{N}(\hat\bfx^{\smalloplus})  \mathbb{N}(\hat\bfx^{\smalloplus})^{\dagger}\right).
\end{equation}
Observing that the following equation holds:
\begin{equation}
    \setlength{\arraycolsep}{2pt}
    \mathbb{N}(\hat\bfx^{\smallominus}) - \mathbb{N}(\hat\bfx^{\smalloplus}) = \big[
        \bfZo_{N\times 3} \quad \underbrace{\bfNg(\hat{\bfx}^{\smallominus}) - \bfNg(\hat{\bfx}^{\smalloplus})}_{\boldsymbol{\alpha}}\\
    \big],
\end{equation} then the second term of \eqref{equ:usa2} can be simplified as the matrix multiplication of a column vector and a row vector:
\begin{equation}
    \bfT = \bfI + \boldsymbol{\alpha} \boldsymbol{\beta}^\top,
\end{equation}
where $\boldsymbol{\beta}^\top$ 
 is the fourth row of $\bfN(\hat\bfx^{\smalloplus})^{\dagger}$. When implementing direct transformation, we can fully utilize the structure of the transformation matrix to accelerate the computation, as detailed in Appendix \ref{appendix:trans_cov}.

\begin{remark}
The direct transformation-based USA possesses a high degree of design flexibility. Although \eqref{equ:usa2} provides a general solution for the direct transformation, this solution is based on the constraint \eqref{equ:criteria}. In practice, we can incorporate other constraints according to the specific requirements.

\end{remark}

\textbf{Indirect transformation:}
According to \eqref{equ:N_alignment}, the unobservable subspaces before and after USA are the values of the function $\mathbb{N}(\cdot)$  evaluated at the previous and current estimates, respectively. Thus, an intuitive way to align the unobservable subspace is to design an auxiliary matrix $\mathbb{T}(\cdot)$ that transforms the state-dependent unobservable subspace into a constant unobservable subspace:
\begin{equation}
    \mathbb{T}({\hat\bfx}) \mathbb{N}({\hat\bfx}) = \bfN_{\text{const}},
    \label{equ:const}
\end{equation}
where $\mathbb{T}({\hat\bfx}) \in \mathbb{R}^{N\times N}$, and $\bfN_{\text{const}}$ is a constant matrix with full column rank. Then, the transformation matrix $\bfT$ can be obtained as follows:
\begin{equation}
    \bfT= \mathbb{T}(\hat{\bfx}^{\smallominus})^{-1} \mathbb{T}(\hat{\bfx}^{\smalloplus}) .
    \label{equ:usa1}
\end{equation}
The indirect transformation is also a special solution to \eqref{equ:N_alignment} with $\bfM = \mathbb{T}(\hat{\bfx}^{\smallominus})^{-1} \mathbb{T}(\hat{\bfx}^{\smalloplus})$.
It should be noted that the auxiliary matrix $\mathbb{T}$ not only exists, but is also not unique \cite{tian2025teskf}. Consequently, the transformation matrix $\bfT$ is also not unique.
Here, we provide the design of the auxiliary matrix for the standard estimator:
\begin{equation}
    \mathbb{T}({\hat\bfx}) = 
    \left[ 
        \begin{array}{cccccc}
            {^G}{\hat \bfR}_{I_\tau}&\bfZo&\bfZo&\bfZo&\bfZo&\bfZo\\
            {[{^G}{\hat\bfp}_{I_\tau}]}_\times {^G}{\hat\bfR}_{I_\tau}&\bfI_3 &\bfZo&\bfZo&\bfZo&\bfZo\\
            {[{^G}{\hat\bfv}_{I_\tau}]}_\times {^G}{\hat\bfR}_{I_\tau}&\bfZo &\bfI_3 &\bfZo&\bfZo&\bfZo\\
            \bfZo&\bfZo&\bfZo& \bfI_6&\bfZo&\bfZo\\
            \bfZo&\bfZo&\bfZo&\bfZo& \mathbb{W}(\hat\bfx) &\bfZo\\
            \mathbb{F}({\hat\bfx}) {^G}{\hat\bfR}_{I_\tau}  &\bfZo&\bfZo&\bfZo&\bfZo& \bfI_{3m}
        \end{array}
    \right],
    \label{equ:Transform}
\end{equation}
with 
\begin{equation}
    \mathbb{W}({\hat\bfx}) =
    \setlength{\arraycolsep}{1pt}
    \begin{bmatrix}
        \ddots \\
        & \begin{matrix}
            {^G}{\hat\bfR}_{I_i} &\bfZo_{3\times 3} \\
            [{^{G}}{\hat\bfp}_{I_i}]_\times {^G}{\hat\bfR}_{I_i}&\bfI_{3}
        \end{matrix}  \\
        && \ddots
    \end{bmatrix}\in \mathbb{R}^{6n \times 6n},
\end{equation}
\begin{equation}
    \mathbb{F}({\hat\bfx}) = 
    \setlength{\arraycolsep}{0.02cm}
    \begin{array}{l}
        \big[  \cdots  \hspace{0.1cm} [{^G}{\hat\bfp}_{f_{j}}]_\times^\top  \hspace{0.1cm} \cdots \big]^\top 
    \end{array} \in \mathbb{R}^{3m \times 3}.
\end{equation}
Substituting \eqref{equ:Transform} into \eqref{equ:usa1}, we get the transformation matrix. 
Unlike \cite{chen2024visualinertial,tian2025teskf} using \emph{globally-defined orientation error}, the auxiliary matrix \eqref{equ:Transform} is designed based on the \emph{locally-defined orientation error} \cite[Chapter 5]{eskf}, which is more commonly utilized in existing VINS or SLAM literature. Additionally,
the sub-auxiliary matrix $\mathbb{W}({\hat\bfx})$, corresponding to the variables within the sliding window $\bfx_{\mathcal{W}}$, has not been considered in existing literature.

\begin{remark}
Applying the auxiliary matrix $\mathbb{T}({\hat\bfx})$ to the linearized error-state system \eqref{equ:less-propagation} and \eqref{equ:less-measurement} yields the linearized error-state system for the transformed error-state $\tilde{\mathbf{x}}^* \triangleq \mathbb{T}(\hat{\mathbf{x}})\tilde{\mathbf{x}}$ \cite{tian2025teskf}. The unobservable subspace within the transformed error-state is precisely $\mathbf{N}_{\text{const}}$, which also constitutes the unobservable subspace of the RI-EKF.  Therefore, from this perspective, the indirect transformation establishes a connection between the standard estimator and the RI-EKF.
     \label{remark:ri_tekf}
\end{remark}

\begin{remark}
\label{remark:affine}
        There is a procedural similarity between the transformation-based USA and the methods presented in \cite{hao2025transformationbased,song2024affine}, in that both approaches transform the covariance matrix to ensure consistency.  Nevertheless, three primary distinctions exist between them.
        \begin{itemize}
            \item \emph{Theoretical foundation:} In \cite{hao2025transformationbased,song2024affine}, covariance transformation is motivated by the computational equivalence between transforming the covariance and performing estimation based on transformed error‑states. Our method, by contrast, interprets the covariance transformation as a change in the evaluation point of the unobservable subspace, thereby mitigating inconsistency by aligning the evaluation point.
            \item \emph{Transformation scope:} The approaches in \cite{hao2025transformationbased,song2024affine} primarily apply the covariance transformation after the correction step. As will be shown later, we find that additional estimation steps within VINS estimators, such as delayed feature initialization, also benefit from incorporating the transformation to maintain consistency.
            \item \emph{Transformation design:} The approaches in \cite{hao2025transformationbased,song2024affine} introduce an auxiliary matrix to render the unobservable subspace constant. In contrast, our direct transformation avoids the need for such a matrix or a constant subspace, offering a more flexible design. 
        \end{itemize}
\end{remark}

\subsection{Re-evaluation-based USA}
The analyses in the previous section indicate that the evaluation point of the estimator's unobservable subspace depends on the linearization points of the Jacobians involved in the information matrix calculation. Therefore, the evaluation point can be aligned by re-evaluating these Jacobians at the current estimate. Like the transformation-based USA, re-evaluation solely works on the uncertainty without directly changing the state vector. Specifically, upon completing the step, we re-evaluate the information matrix as follows:
\begin{equation}
       \bfLa^{\smalloplus} = (\bfLa^{\smallominus})^* + {\bfH^*}^\top \bfR^{-1} \bfH^*,
\end{equation}
where the Jacobians in $(\bfLa^{\smallominus})^*$ and $\bfH^*$ are linearized at the current estimate.

Re-evaluation can simultaneously eliminate misalignment and reduce linearization error within the information matrix. However, we recommend re-evaluation-based USA as a complementary tool to transformation-based USA, rather than a standalone solution. This is because re-evaluating the prior information $(\bfLa^{\smallominus})^*$ is often computationally prohibitive, as it may require re-accessing Jacobians accumulated from the beginning, making the process unsustainable for real-time applications. Nevertheless, in certain cases, only the Jacobians in ${\bfH^*}^\top \bfR^{-1} \bfH^*$ require re-evaluation. For instance, during delayed feature initialization (Section \ref{sec:feature_initialization}), the prior $\bfLa^{\smallominus}$ contains no information about newly initialized features, making its re-evaluation unnecessary. Therefore, re-evaluation is the preferred approach for optimal performance when the computational burden is manageable. Otherwise, a transformation-based USA should be employed to balance accuracy and efficiency.

\remark{
    Re-evaluation and FEJ
    represent two fundamentally distinct approaches, despite both modifying the linearization points of the Jacobians.  From the perspective of the unobservable subspace evaluation point, 
    the unobservable subspace in re-evaluation is Aligned, while the unobservable subspace in FEJ is Misaligned and always evaluated at the initial estimate. 
    In terms of execution, re-evaluation modifies only the information matrix $\bfLa^{\smalloplus}$ after the estimation step is completed, whereas FEJ alters the Jacobians during the estimation step, thereby directly affecting both the information matrix and the state estimate 
$\hat\bfx^{\smalloplus}$
.
    }

\subsection{Implementation}

USA can be conveniently integrated into existing VINS estimators as a corrective module. It operates solely on the information (or covariance) matrix at specific estimation steps. Since this matrix serves as the output of an upstream step and the input to a downstream step, USA's intervention requires no changes to the computational procedures of any estimation steps. This design enables USA to perform as a non-intrusive refinement within the overall estimation pipeline.

USA does not increase the overall computational complexity of the estimator. 
Most VINS estimators represent uncertainty using covariance matrices, and the transformation-based USA can be efficiently implemented through them.
By leveraging the inherent sparsity of the matrices, both the direct and indirect transformations have a computational complexity of $O(N^2)$, as detailed in Appendix \ref{appendix:trans_cov}.  For comparison, the computational complexity of the correction itself is $O(N^3)$.
For the re-evaluation-based USA, its computational cost is dependent on the specific application scenario. In its application to delayed feature initialization, its computational load is even lower than that of the transformation (see Fig. \ref{fig:time_accuracy}).

\section{Application on Sliding Window Filter}

\label{sec:VINS_Estimator}
This section utilizes the proposed USE and USA to identify and eliminate the observability misalignments within the SWF framework \cite{genevaOpenVINSResearchPlatform2020,li2014visual}. 
SWF comprises several estimation steps beyond the prediction and SLAM correction steps discussed previously, including MSCKF correction, delayed feature initialization, state augmentation, and marginalization. As state augmentation and marginalization do not introduce observability misalignment, their detailed analysis is deferred to Appendix \ref{app:imu_aug}.
In addition to demonstrating the effectiveness of the proposed approaches, the application on SWF also achieves the following three objectives: (i) clarifying that the consistency of MSCKF differs from that of EKF-SLAM, thereby correcting a common misconception in the literature;  (ii) presenting two tailored strategies for scenarios where multiple estimation substeps share a common linearization point; and (iii) illustrating the use of re-evaluation-based USA in delayed feature initialization.

\subsection{MSCKF Correction}
\label{sec:msckf}

The features in SWF are categorized into two types: SLAM features and MSCKF features. 
If a feature is continuously observable over a certain number of frames, it is considered a SLAM feature and added to the system state through the delayed feature initialization.
Otherwise,  it is treated as an MSCKF feature, and its measurements are utilized to formulate a multi-state constraint. 
As discussed in Section \ref{sec:status_transition}, the unobservable subspace after the SLAM correction is Misaligned. In the following, we will show that the MSCKF correction also leads to misalignment.
For simplicity, we use \((\hat\bfx^{\smallominus},\bfLa^{\smallominus})\) and \((\hat\bfx^{\smalloplus},\bfLa^{\smalloplus})\) to denote the estimates before and after the MSCKF correction and consider only one MSCKF feature, though the approach easily extends to multiple features.
Let $\bfz \in \mathbb{R}^{2n}$ be the stack of the measurements of the MSCKF feature observed by $n$ cloned IMU poses.
The linearized measurement model is:
\begin{equation}
    \tilde\bfz = \bfH_x \tilde\bfx^{\smallominus} + \bfH_{f} {^G}\tilde{\bfp}_{f}^{\smallominus} + \boldsymbol{\epsilon}.
    \label{equ:sw_measurement} 
\end{equation}
where $(\bfH_x,\bfH_{f})$ is evaluated at the current best estimate $(\hat{\bfx}^{\smallominus},{^G}\hat{\bfp}_{f}^{\smallominus})$, and ${^G}\hat{\bfp}_{f}^{\smallominus}$ is the estimate of the feature position obtained through triangulation.
Since the MSCKF feature ${^G}{\bfp}_{f}$ is not included in the state vector, the above linearized measurement model cannot be directly used to correct the system state.
To address this issue, \cite{mourikisMultiStateConstraintKalman2007} applies an orthogonal matrix $\bfQ = \begin{bmatrix}
        \bfQ_1^\top &     \bfQ_2^\top
    \end{bmatrix}^\top$ with $\bfQ_2 \bfH_f = \bfZo$ on the measurement model to separate it into the following two subsystems:
\begin{align}
    \begin{bmatrix}
        \bfQ_1   \tilde{\bfz} \\
        \bfQ_2   \tilde{\bfz}  
    \end{bmatrix} &= 
    \begin{bmatrix}
        \bfQ_1 \bfH_x\\
        \bfQ_2 \bfH_x \\
    \end{bmatrix} \tilde\bfx^{\smallominus}  + \begin{bmatrix}
        \bfQ_1 \bfH_{f} \\
        \bfZo_{{(2n-3)}\times 3}\\
    \end{bmatrix}  {^G}\tilde{\bfp}_{f}^{\smallominus} + \begin{bmatrix}
        \bfQ_1  \boldsymbol{\epsilon}\\
        \bfQ_2 \boldsymbol{\epsilon}
    \end{bmatrix}
    \label{equ:measurement_subsystem0}, \\
        \Rightarrow &\left\{
        \begin{array}{l}
            \tilde{\bfz}_1 =  \bfH_{x1}\tilde\bfx^{\smallominus} + 
            \bfH_{f1} {^G}\tilde{\bfp}_{f}^{\smallominus}+
            \boldsymbol{\epsilon}_{1} \\
            \tilde{\bfz}_2 =  \bfH_{x2}\tilde\bfx^{\smallominus}  +  \boldsymbol{\epsilon}_{2} \\
        \end{array}\right. . \label{equ:measurement_subsystem}
\end{align}
The second measurement subsystem is independent of the MSCKF feature and used to perform the MSCKF correction. During MSCKF correction, the information matrix is updated as 
\begin{equation}
    \bfLa^{\smalloplus} = \bfLa^{\smallominus} + {\bfH_{x2}}^\top \bfR_{2}^{-1} \bfH_{x2}.
    \label{equ:MSCKF_correction}
\end{equation}
Noting that 
\begin{equation}
    \begin{bmatrix}
    \bfH_x &
    \bfH_{f}
    \end{bmatrix} \begin{bmatrix}
        \bfN(\hat\bfx^{\smallominus})\\
        \bfN_f({^G}\hat{\bfp}_{f}^{\smallominus})\\
    \end{bmatrix} = \bfZo,
    \label{equ:N_f}
\end{equation}
where $\bfN_f(\bfp):= \begin{bmatrix}
    \bfI_3 & [\bfp]_\times\bfg
\end{bmatrix}, \bfp \in \mathbb{R}^3$, applying $\bfQ_2$ on the left of the above equation yields:
\begin{equation}
    \bfH_{x2} \bfN(\hat\bfx^{\smallominus}) = \bfZo.
    \label{equ:msckf_N}
\end{equation}
Same as \eqref{equ:La_update}-\eqref{equ:La_N_Acc} in the SLAM correction, by combining \eqref{equ:MSCKF_correction} and \eqref{equ:msckf_N}, we can see that the unobservable subspace is Misaligned after the MSCKF correction:
\begin{equation}
    \mathbf{N}^{\smalloplus} =  \bfN(\hat\bfx^{\smallominus}).
    \label{equ:msckf_misalignment}
\end{equation}
Since MSCKF correction causes misalignment, as well as the SLAM correction, the transformation-based USA is required to maintain the consistency of SWF.

\begin{remark}
    Previous works have generally assumed that MSCKF inherits the consistency property of EKF-SLAM \cite{li2013HighprecisionConsistentEKFbased,huai2022robocentric}; however, our analysis suggests that this assumption may not hold in all cases.
    Comparing \eqref{equ:N_f}-\eqref{equ:msckf_misalignment} with \eqref{equ:unobservable_subspace}, one can see that if a feature is only used to perform the MSCKF correction, the state of the MSCKF feature will not be included in the unobservable subspace. A consequence of this is that if an estimator only performs the MSCKF correction and does not maintain any SLAM features, its unobservable subspace will be composed of the top two blocks of \eqref{equ:unobservable_subspace}.
    Therefore, we only need to consider the top two blocks of \eqref{equ:unobservable_subspace} rather than the entire unobservable subspace when assessing the consistency of the MSCKF estimator.
    According to Theorem \ref{the:2}, maintaining the top two blocks being constant is sufficient to ensure the consistency of the MSCKF estimator. The validation for this finding will be given in Section \ref{sec:sim_USE}.
    \label{rmk:msckf}
\end{remark}

\subsection{Delayed Feature Initialization}
\label{sec:feature_initialization}
If a feature ${^G}{\bfp}_f$ has been observed for a certain number of frames, it will be added to the system state and initialized as a SLAM feature. 
Following MSCKF correction, delayed feature initialization also relies on the orthogonal operation to separate the measurement model into two subsystems, i.e.,  \eqref{equ:measurement_subsystem}. Differently, there are two substeps in the initialization, as both subsystems are also utilized sequentially. 

The delayed feature initialization represents a class of estimation steps, where the steps are divided into multiple substeps that share the same linearization points. Here, we present two strategies to eliminate the misalignment caused by these substeps: performing USA in batch or separately, as shown in Fig. \ref{fig:comparePostCorrect}. The former is a general strategy for multiple substep cases, while the latter is specific to delayed feature initialization but is more accurate.
\begin{figure}[thbp]
    \centering
    \includegraphics[width=0.48\textwidth]{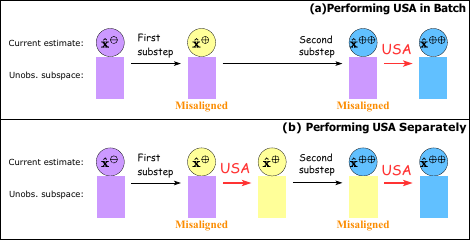} 
    \caption{Different strategies for delayed feature initialization.
    (a) Applying USA in batch once after the composite step of the initialization.
     (b) Applying USA separately after both the first and the second substeps.
    }
    \label{fig:comparePostCorrect}
\end{figure}

\textbf{Performing USA in batch:}
The first substep adds the feature to the state vector based on the first subsystem of \eqref{equ:measurement_subsystem}. Here, we use  \((\hat\bfx^{\smallominus},\bfLa^{\smallominus})\) and  \((\hat\bfx^{\smalloplus},\bfLa^{\smalloplus})\) to denote the estimate before and after the first substep, respectively:
\begin{align}
    \hat{\bfx}^{\smalloplus} &= \begin{bmatrix}
        \hat{\bfx}^{\smallominus}\\ 
        {^G}\hat{\bfp}_{f}^{\smallominus}
    \end{bmatrix}\oplus {\bfLa^{\smalloplus}}^{-1} {\bfH_{\text{1}}}^\top \bfR_1^{-1} \tilde{\bfz}_1,
    \label{equ:fi_x_update}\\
        \bfLa^{\smalloplus} &= \begin{bmatrix}
        \bfLa^{\smallominus} & \bfZo \\
        \bfZo & \bfZo 
    \end{bmatrix} + {\bfH_{\text{1}}}^\top\bfR_1^{-1} {\bfH_{\text{1}}},
    \label{equ:fi_La_update}
\end{align}
where $\bfH_{\text{1}} =  \begin{bmatrix}
    \bfH_{x1} &
    \bfH_{f1}
\end{bmatrix}$ is evaluated at $(\hat{\bfx}^{\smallominus},{^G}\hat{\bfp}_{f}^{\smallominus})$. 
The second substep enhances the state estimate based on the second subsystem.
Its resulting estimate is denoted as  \((\hat\bfx^{\smalloplus\smalloplus},\bfLa^{\smalloplus\smalloplus})\):
\begin{align}
    \hat{\bfx}^{\smalloplus\smalloplus} &= \hat{\bfx}^{\smalloplus}\oplus {\bfLa^{\smalloplus\smalloplus}}^{-1} {\bfH_{\text{2}}}^\top \bfR_2^{-1} \tilde{\bfz}_2,\\
    \bfLa^{\smalloplus\smalloplus} &= \bfLa^{\smalloplus} + {\bfH_{2}}^\top \bfR_{2}^{-1} \bfH_{2},
    \label{equ:La_updata_pp}
\end{align}
where $\bfH_{2} = \begin{bmatrix}
        \bfH_{x2} & \bfZo_{(2n-3)\times 3 }\end{bmatrix}$ is also evaluated at $(\hat{\bfx}^{\smallominus},{^G}\hat{\bfp}_{f}^{\smallominus})$.
Since the first and second substeps share the same linearization point, the correction of the information matrix during these two substeps, i.e., \eqref{equ:fi_La_update} and \eqref{equ:La_updata_pp}, can be combined as follows:
\begin{equation}
    \bfLa^{\smalloplus\smalloplus} = \begin{bmatrix}
        \bfLa^{\smallominus} & \bfZo \\
        \bfZo & \bfZo 
    \end{bmatrix} + {{\bfH}}^\top\bfR^{-1} {\bfH}
    \label{equ:La_section_step}
\end{equation}
where $
    \bfR = \text{Diag}(\bfR_1,\bfR_2) $ and
    $ \bfH = \begin{bmatrix}
        {\bfH_1}^\top 
        {\bfH_2}^\top
    \end{bmatrix}^\top$ are evaluated at $(\hat{\bfx}^{\smallominus},{^G}\hat{\bfp}_{f}^{\smallominus})$.

    Equations \eqref{equ:fi_La_update} and \eqref{equ:La_section_step} share the same form as \eqref{equ:La_update}, which implies that the unobservable subspace remains unchanged across the substeps. More precisely, it is consistently evaluated at the initial estimate:
\begin{equation}
    \mathbf{N}^{\smalloplus\smalloplus}= \mathbf{N}^{\smalloplus}=  \begin{bmatrix}
        \bfN(\hat\bfx^{\smallominus})\\
        \bfN_f({^G}\hat{\bfp}_{f}^{\smallominus})\\
    \end{bmatrix}.
    \label{equ:feature_init}
\end{equation}
This observation suggests that USA should be performed in batch for multiple substep cases, as shown in Fig. \ref{fig:comparePostCorrect}(a). Specifically, the unobservable subspace is directly aligned from the initial estimate $(\hat\bfx^{\smallominus},{^G}\hat{\bfp}_{f}^{\smallominus})$ to the final estimate $\hat\bfx^{\smalloplus\smalloplus}$.

\textbf{Performing USA separately:}
For delayed feature initialization,
performing USA separately after both substeps also proves effective, as illustrated in Fig. \ref{fig:comparePostCorrect}(b). 

Since $ \bfLa^{\smallominus}$ does not contain any information about the feature, only the feature estimate is corrected after the first substep, that is, \eqref{equ:fi_x_update} is equivalent to:
\begin{equation}
    \hat{\bfx}^{\smalloplus} = \begin{bmatrix}
        \hat{\bfx}^{\smallominus} \\
        {^G}\hat{\bfp}_{f}^{\smallominus} +\bfH_{f1}^{-1}\tilde{\bfz}_1 \\   
    \end{bmatrix},
    \label{equ:update0}
\end{equation}
which indicates that only the sub-unobservable subspace corresponding to the newly initialized feature, $\bfN_f({^G}\hat{\bfp}_{f}^{\smallominus})$, is Misaligned. Thus, after applying USA for the first substep, the unobservable subspace is given as follows:
\begin{equation}
    \mathbf{N}^{\smalloplus}= \bfN(\hat\bfx^{\smalloplus})= \begin{bmatrix}
        \bfN(\hat\bfx^{\smallominus})\\
        \bfN_f({^G}\hat{\bfp}_{f}^{\smalloplus})\\
        \end{bmatrix}.
        \label{equ:64}
\end{equation}

As for the Jacobian of the second substep, we have 
\begin{align}
     &\begin{bmatrix}
        \bfH_{x2} & \bfZo_{(2n-3)\times 3 }\end{bmatrix} \begin{bmatrix}
        \bfN(\hat\bfx^{\smallominus})\\
        \bfN_f({^G}\hat{\bfp}_{f}^{\smalloplus})\\
        \end{bmatrix} \overset{\eqref{equ:msckf_N}} {=} \bfZo, \\
        \Rightarrow &\bfH_{2} \bfN(\hat\bfx^{\smalloplus}) = \bfZo,
        \label{equ:66}
\end{align}
which implies that after the second substep, the unobservable subspace remains $\bfN(\hat\bfx^{\smalloplus})$. This allows us to apply a transformation-based USA to align it from $\hat\bfx^{\smalloplus}$ to $\hat\bfx^{\smalloplus\smalloplus}$. 
Note that \eqref{equ:66} provides the key condition enabling USA to be performed separately in delayed feature initialization. In other cases involving multiple substeps, this condition may not hold.

Importantly, the misalignment caused by the first substep can be efficiently eliminated using the re-evaluation-based USA. Specifically, after correcting the estimate using \eqref{equ:update0}, we replace \eqref{equ:fi_La_update} with the following equation:
\begin{equation}
    \bfLa^{\smalloplus} = \begin{bmatrix}
        \bfLa^{\smallominus} & \bfZo \\
        \bfZo & \bfZo 
    \end{bmatrix} + {\bfH_{\text{1}}^{*}}^\top\bfR_1^{-1} {\bfH_{\text{1}}^{*}},
    \label{equ:fi_La_update2}
\end{equation}
where $\bfH_{\text{1}}^{*} =  \begin{bmatrix}
    \bfH_{x1}^{*} &
    \bfH_{f1}^{*}
\end{bmatrix}$ is evaluated at the corrected estimate $(\hat{\bfx}^{\smallominus},{^G}\hat{\bfp}_{f}^{\smalloplus})$. Note that we do not need to re-evaluate the prior information matrix $\bfLa^{\smallominus}$  because it does not contain any information about the newly initialized feature.  Since the re-evaluation-based USA after the first substep can significantly reduce the linearization error of uncertainty, performing USA separately yields better performance.

\remark{
The application on SWF powerfully demonstrates the high precision of the proposed approaches.
USE initially enables precise localization of the observability misalignment down to the level of sub-unobservable subspaces. Subsequently, USA acts as a “scalpel”, accurately rectifying the identified misalignment while leaving the remaining component unaffected.
\label{remark:separately}
}

\section{Simulation and Experiment Results}
In this section, we evaluate the proposed approaches from four perspectives. First, we validate that the USE framework can rigorously assess the consistency of SWF. Second, we compare different USA strategies for delayed feature initialization and select the more effective ones for subsequent tests. Third, we benchmark the proposed USA against state‑of‑the‑art methods in terms of accuracy, consistency, and computational efficiency. Finally, we assess the performance of USA in real‑world scenarios.

 To ensure fairness and comparability of the results, all the estimators in these tests are implemented on the OpenVINS platform.
Three OpenVINS configurations are used during the tests: (1) SLAM mode, which performs SLAM correction and delayed feature initialization; (2) MSCKF mode, which performs MSCKF correction only; and (3) Hybrid mode, which executes all of the above steps. Unless stated otherwise, the maximum numbers of MSCKF and SLAM features are set to 40, and the number of cloned IMU poses is set to 11, following OpenVINS' default settings. 
The first three parts of the evaluation are conducted in a simulated environment using the OpenVINS built-in simulator. During the simulations, the robot follows two trajectories recorded in OpenVINS: Udel-gore and Udel-neig. Udel-gore represents the trajectory of a handheld device operating in an indoor environment, whereas Udel-neig corresponds to the trajectory of a car in an outdoor setting. These two trajectories encompass a broad spectrum of movements, including rapid rotations and translations, making them widely accepted benchmarks for evaluating the performance of the VINS estimator.
Table \ref{tab:sim_params} provides a summary of the basic parameters for the sensors.
\begin{table}[htp]
    \caption{Simulator basic parameters}
    \centering
    \setlength\tabcolsep{4.5pt}
    \begin{tabular}[]{cccc}
            \toprule
            { \textbf{Parameter}} & \textbf{Value} & \textbf{Parameter} & \textbf{Value} \\
            \midrule
            Accel. White Noise    & 2.00e-03       & Gyro. White Noise  & 1.70e-04       \\
            Accel. Random Walk    & 3.00e-03       & Gyro. Random Walk  & 2.00e-05       \\
            IMU Freq.          & 200    & Cam Freq.          & 10  \\
            Cam Number         & Mono   & Cam Noise          & 2 pixel       \\
            Cam Resolution          & 720 $\times$ 480       &Cam $f_x$\&$f_y$ &(459, 457)                     \\
            \bottomrule
    \end{tabular}
    \label{tab:sim_params}
\end{table}

\begin{figure}
    \centering 
    \includegraphics[width=0.485\textwidth]{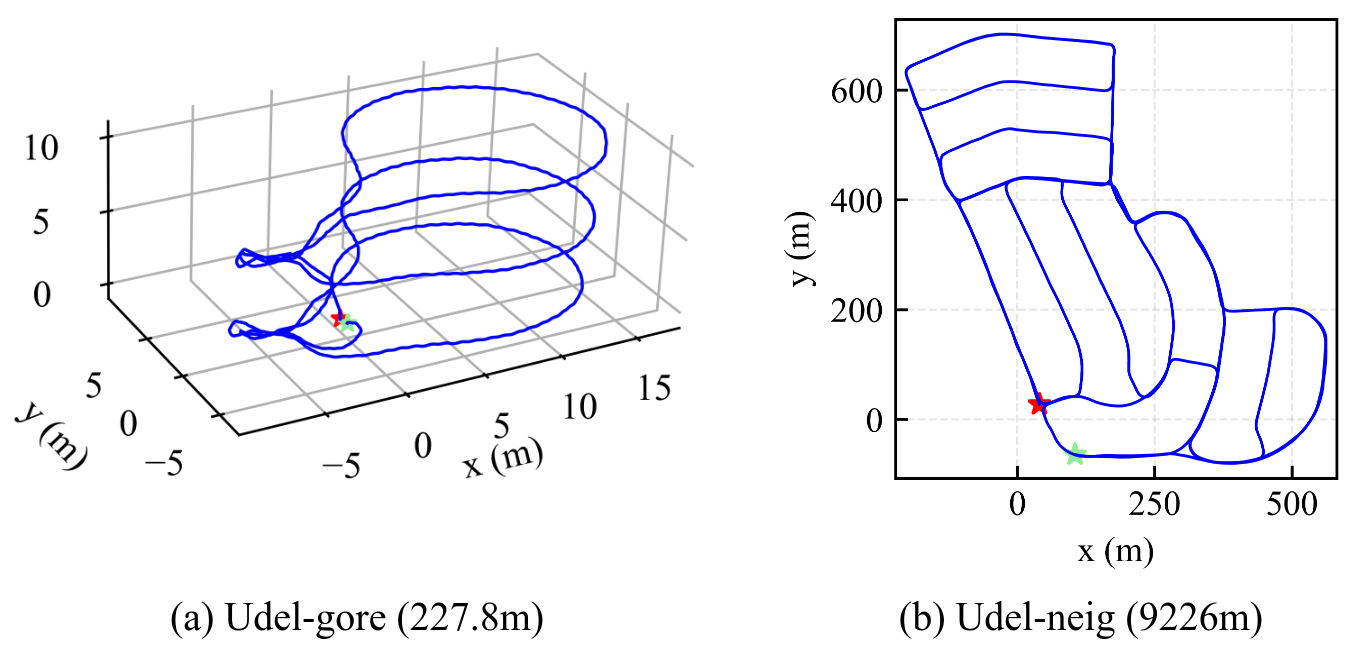}
    \caption{
        Trajectories used in the simulation. The green star represents the starting point, while the red star indicates the endpoint of the trajectory.
        }
        \label{fig:traj_sim}
\end{figure}
The following metrics are utilized for assessment: Root Mean Square Error (RMSE) to evaluate trajectory accuracy and Normalized Estimation Error Squared (NEES) to assess system consistency. A lower RMSE indicates higher estimation accuracy, while an NEES value closer to 1 reflects better consistency.

\subsection{USE Validation}
\label{sec:sim_USE} 
The validation of USE includes two parts. 
First, we demonstrate that USE accurately assesses MSCKF consistency and correct the previous misunderstanding that its consistency is the same as that of EKF-SLAM.
Then, we show that not only corrections, but also delayed feature initialization lead to misalignment.

\textbf{MSCKF consistency:}
The proposed USE can accurately assess the consistency of MSCKF and demonstrates that it is not the same as that of EKF-SLAM. To validate our claim, we consider three different estimators. Each estimator executes in two modes: MSCKF mode and SLAM mode (i.e., EKF-SLAM).  The three estimators are as follows:
\begin{itemize}
    \item Std: The standard estimator as defined in Section \ref{sec:system_model}. 
    \item RI-EKF: Models the extended IMU pose together with the $m$ landmarks on $\mathbf{SE}_{2+m}(3)$ \cite{wuInvariantEKFVINSAlgorithm2017}.
    \item PRI: An estimator with a partially right-invariant representation, in which the extended IMU pose and landmarks are modeled on $\mathbf{SE}_{2}(3)$  and $\mathbb{R}^{3m}$, respectively \cite{tsao2023AnalyticIMUPreintegration}.
\end{itemize}

\begin{table}[t]
    \caption{Consistency properties of different estimators in MSCKF mode and SLAM mode (i.e., EKF-SLAM)}
    \centering
    \setlength\tabcolsep{4.5pt}
    \begin{tabular}[]{cccc}
            \toprule
            { \textbf{Estimator}} & EKF-SLAM & MSCKF & Same Consistency \\
            \midrule
            Std & Inconsistent & Inconsistent & \Checkmark \\
            RI-EKF & Consistent & Consistent & \Checkmark \\
            PRI & Inconsistent & Consistent & \XSolidBrush \\
            \bottomrule
    \end{tabular}
    \label{tab:consistency}
\end{table}

Table~\ref{tab:consistency} presents the consistency properties of these estimators, followed by a detailed analysis of each estimator.
in the SLAM mode, the unobservable subspace of Std is given by \eqref{equ:unobservable_subspace}, while those of RI-EKF and PRI are $\mathbf{N}_{\text{RI-EKF}}$ and $ \mathbf{N}_{\text{PRI}}$, respectively:
\begin{equation}
\scriptstyle
    \bfN_{\text{RI-EKF}} (\bfx) = 
\setlength{\arraycolsep}{4pt}
     \left[\begin{array}{cc}
    \bfZo_{3\times 3} &  -\bfg            \\
    \bfI_3            & \bfZo_{3\times 1}\\
    \bfZo_{3\times 3} & \bfZo_{3\times 1}\\
    \bfZo_{3\times 3} & \bfZo_{3\times 1}\\
    \bfZo_{3\times 3} & \bfZo_{3\times 1}\\
            \hdashline
    \vdots & \vdots \\
    \bfZo_{3\times 3} & -   \bfg \\
    \bfI_3 & \bfZo_{3\times 1}\\
    \vdots & \vdots \\
    \hdashline
    \vdots & \vdots \\
    \bfI_3 & \bfZo_{3\times 1} \\
    \vdots & \vdots \\
\end{array}\right],    \hspace{3pt}     \bfN_{\text{PRI}}(\bfx) = 
\setlength{\arraycolsep}{1pt}
\left[\begin{array}{cc}
    \bfZo_{3\times 3} & - \bfg            \\
    \bfI_3            & \bfZo_{3\times 1}\\
    \bfZo_{3\times 3} & \bfZo_{3\times 1}\\
    \bfZo_{3\times 3} & \bfZo_{3\times 1}\\
    \bfZo_{3\times 3} & \bfZo_{3\times 1}\\
            \hdashline
    \vdots & \vdots \\
    \bfZo_{3\times 3} & -\bfg \\
    \bfI_3 & \bfZo_{3\times 1}\\
    \vdots & \vdots \\
        \hdashline
        \vdots & \vdots \\
        \bfI_3 & [{^G}{\bfp}_{f_{j}}]_\times \bfg\\
        \vdots & \vdots \\
        \end{array}\right].
    \label{equ:unobservable_subspace_ri}
\end{equation}
Since the unobservable subspaces of Std and PRI depend on states, misalignments would happen in the SLAM mode, leading to inconsistency. 
According to Remark \ref{rmk:msckf}, the unobservable subspaces in the MSCKF mode are composed of the top two blocks of those of the SLAM mode. 
Because the top two blocks of $\mathbf{N}_{\text{PRI}}$ are state-independent, misalignments would not happen in the PRI in the MSCKF mode. This demonstrates that the earlier view that MSCKF and EKF-SLAM share the same consistency is not correct.

We also validate our conclusion regarding the consistency of MSCKF through Monte Carlo simulations.
Fig. \ref{fig:msckf_rmse} and Fig. \ref{fig:msckf_nees} show the results of 100 runs.
Except for PRI, the SLAM mode exhibits higher accuracy than the MSCKF mode, as the SLAM mode utilizes more information. The “anomalous” performance of PRI is expected because its MSCKF mode is consistent, whereas the SLAM mode is inconsistent.

\begin{figure}[t]
    \centering
    \includegraphics[width=0.485\textwidth]{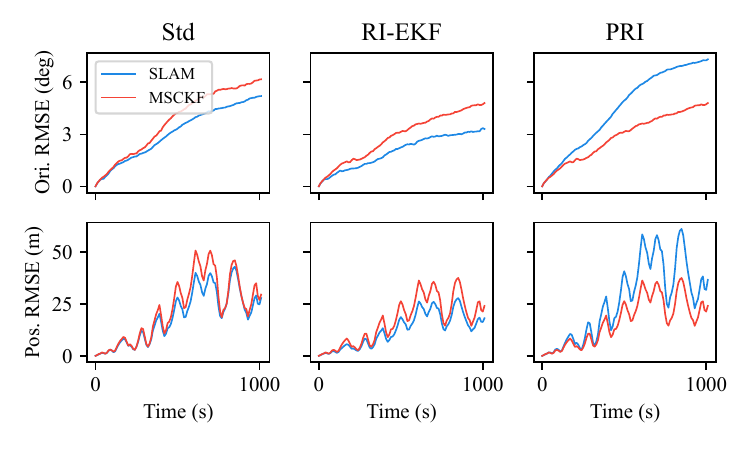}
    \caption{The orientation and position RMSE of different estimators in the MSCKF and SLAM modes.
    The accuracy of the PRI SLAM is worse than that of the PRI MSCKF, because the former is inconsistent, while the latter is consistent.}
    \label{fig:msckf_rmse}
\end{figure}
\begin{figure}[t]
    \centering 
    \includegraphics[width=0.485\textwidth]{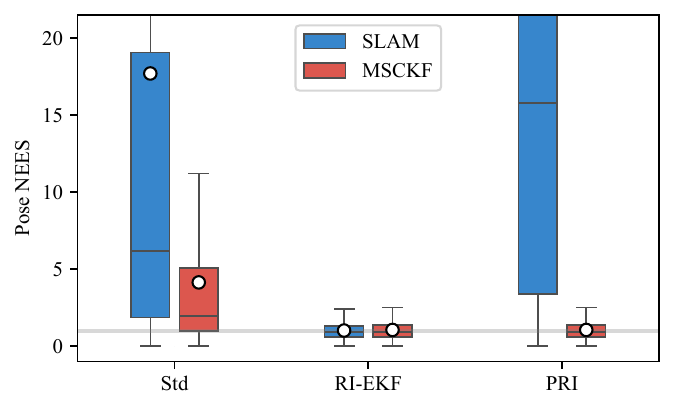}
    \caption{The pose NEES of different estimators in the MSCKF and SLAM modes. The gray line indicates the expected NEES value: 1.  The mean NEES value of RI-EKF SLAM, RI-EKF MSCKF, and PRI MSCKF are approximately 1, indicating that these estimators are consistent. The others are inconsistent.}
    \label{fig:msckf_nees}
\end{figure}

\textbf{Inconsistency source:}
We categorize the misalignments in SWF into two classes: those caused by corrections and those caused by delayed feature initialization.
By eliminating certain classes of misalignment, we can observe the corresponding changes in the estimator's consistency.
Four cases are presented in Table \ref{tab:ab2}, where Case 1 serves as the baseline without USA
and Case 4 eliminates all misalignments.
During this test, all estimators are executed in SLAM mode, and all alignments in the correction and initialization steps are performed using the indirect transformation. Executing the estimators in other modes or using alternative USA methods does not affect the conclusions and is therefore omitted.

\begin{table}[!ht]
    \caption{Average orientation (deg) / position (m) RMSE and NEES of 100 Monte Carlo runs on Udel-gore.}
    \centering
    \begin{threeparttable}          
        \setlength{\tabcolsep}{3.5pt}
    \centering
        \begin{tabular}[]{cccccccc}
            \toprule
            \multirow{2}{*} {Cases}&  \multirow{2}{*} { Corrections} & \multirow{2}{*} {\makecell{Delayed \\ Feat. Init.}} &\multirow{2}{*}  {RMSE}  &\multirow{2}{*}  {NEES} \\
           & \\
            \midrule
           1& -              & -   & 0.823  / 0.199  & 13.55 / 1.702 \\
           2& -              & Aligned   &  0.807  / 0.197  & 12.95 / 1.667  \\
           3& Aligned        & -   & 0.576  / 0.177  & 1.217 / 1.073 \\
           4& Aligned        & Aligned   & 0.450  / 0.167  & 0.916 / 0.981  \\ 
        \bottomrule
    \end{tabular}
\end{threeparttable}       
\label{tab:ab2}
\end{table}

As shown in Table~\ref{tab:ab2}, eliminating arbitrary misalignment improves both consistency and accuracy. Consequently, inconsistencies stem from misalignments induced by both the corrections and the delayed feature initialization. 
However, if the misalignment from the corrections is not addressed, eliminating that from the initialization has only a limited effect, as shown in Fig. \ref{fig:ori_nees_rmse_for_source_inconsistency}. 
This suggests that the misalignment induced by corrections is the main source of inconsistency.
This discrepancy arises because the correction step is executed at least once per frame, whereas feature initialization occurs only once for each SLAM feature.

\begin{figure}
    \centering
    \includegraphics[width=0.99\linewidth]{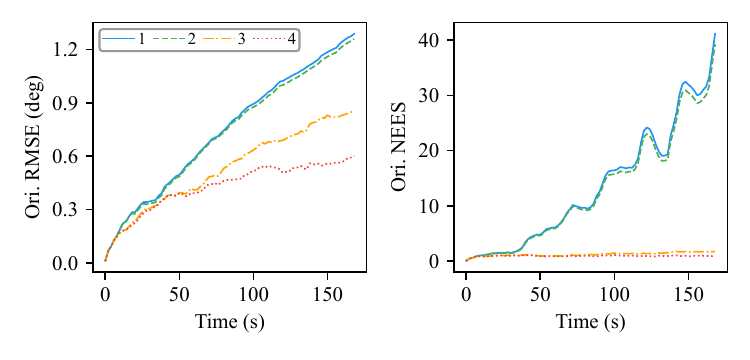}
    \caption{Orientation RMSE and NEES over time of different cases.}
    \label{fig:ori_nees_rmse_for_source_inconsistency}
\end{figure}

\subsection{USA Strategy Comparison}
\label{sec:eliminate_misalignment}
In Section \ref{sec:feature_initialization}, we introduced several strategies to address the misalignment caused by delayed feature initialization, which are summarized in Table \ref{tab:sim_rmse}.
“None” indicates that no USA is performed for the initialization and serves as the baseline (i.e., Case 3 in Table \ref{tab:ab2});
Batch transformation refers to applying transformation-based USA in batch, as depicted in Fig. \ref{fig:comparePostCorrect}(a); Separate transformation denotes performing transformation-based USA separately, as illustrated in Fig. \ref{fig:comparePostCorrect}(b);  Separate re-evaluation is similar to Separate transformation, but it employs re-evaluation-based USA after the first substep.

\begin{table}[!t]
    \centering
    \caption{Average RMSE (deg/meter) and NEES of 100 Monte Carlo runs in different USA strategies for delayed feature initialization}
    \begin{threeparttable}          
            \begin{tabular}[]{clcc}
                    \toprule
                    Dataset& Strategy &RMSE &NEES\\
                    \midrule
                    & None  & 0.576  / 0.177  & 1.217 / 1.073  \\
                    Udel-gore &Batch transformation & 0.450  / \textbf{0.167}  & 0.916 / 0.981   \\
                     (227.8m) &Separate transformation & 0.450  / \textbf{0.167}  & 0.916 / 0.981   \\
                    &Separate re-evaluation &\textbf{0.449}  / \textbf{0.167}  & 0.913 / 0.983   \\
                    \midrule
                    & None   & 4.009  / 26.131  & 8.414 / 7.543   \\

                    Udel-neig& Batch transformation& 2.576  / 16.678  & 1.093 / 1.047 \\
                   (9226m) &Separate  transformation & 2.574  / 16.657  & 1.092 / 1.043   \\
                    &Separate re-evaluation &\textbf{2.326}  / \textbf{15.242}  & \textbf{1.021} / \textbf{0.974}   \\
                    \bottomrule
            \end{tabular}
    \end{threeparttable}       
    \label{tab:sim_rmse}
\end{table}

The results indicate that, except for “None”,  the other methods effectively mitigate the inconsistency issue caused by delayed feature initialization. 
Among them, Batch transformation and Separate transformation yield almost identical results. The difference between them is reasonable, as Separate transformation changes the uncertainty of the first substep, which makes the result of the subsequent substep slightly different from that in Batch transformation.
Notably, Separate re-evaluation achieves the best performance, as it not only aligns the unobservable subspace but also reduces the linearization error of uncertainty. Since performing USA separately yields better performance, we adopt it as the default implementation in subsequent tests.

\subsection{USA Benchmarking}
\label{sec:sim_usa}
We validate the effectiveness of USA by benchmarking it against state‑of‑the‑art methods. Specifically, in addition to Std and RI‑EKF introduced in Section \ref{sec:sim_USE}, we also compare the following estimators:
\begin{itemize}
    \item FEJ: Estimator using the first estimated Jacobians to improve consistency \cite{huangAnalysisImprovementConsistency2008}. 
    \item USA-IT: Estimator with indirect transformation for both corrections and delayed feature initialization.
    \item USA-DT: Similar to USA-IT, but uses direct transformation.
    \item USA-DTR:  Similar to USA-DT, but uses re-evaluation for the first substep of delayed feature initialization.
\end{itemize}
Among these estimators, Std and FEJ are implemented natively in OpenVINS, RI-EKF is implemented by \cite{tian2025teskf}, and USA-IT, USA-DT, and USA-DTR are our own implementations. 
\begin{figure*}[htbp]
    \centering
    \includegraphics[width=1\textwidth]{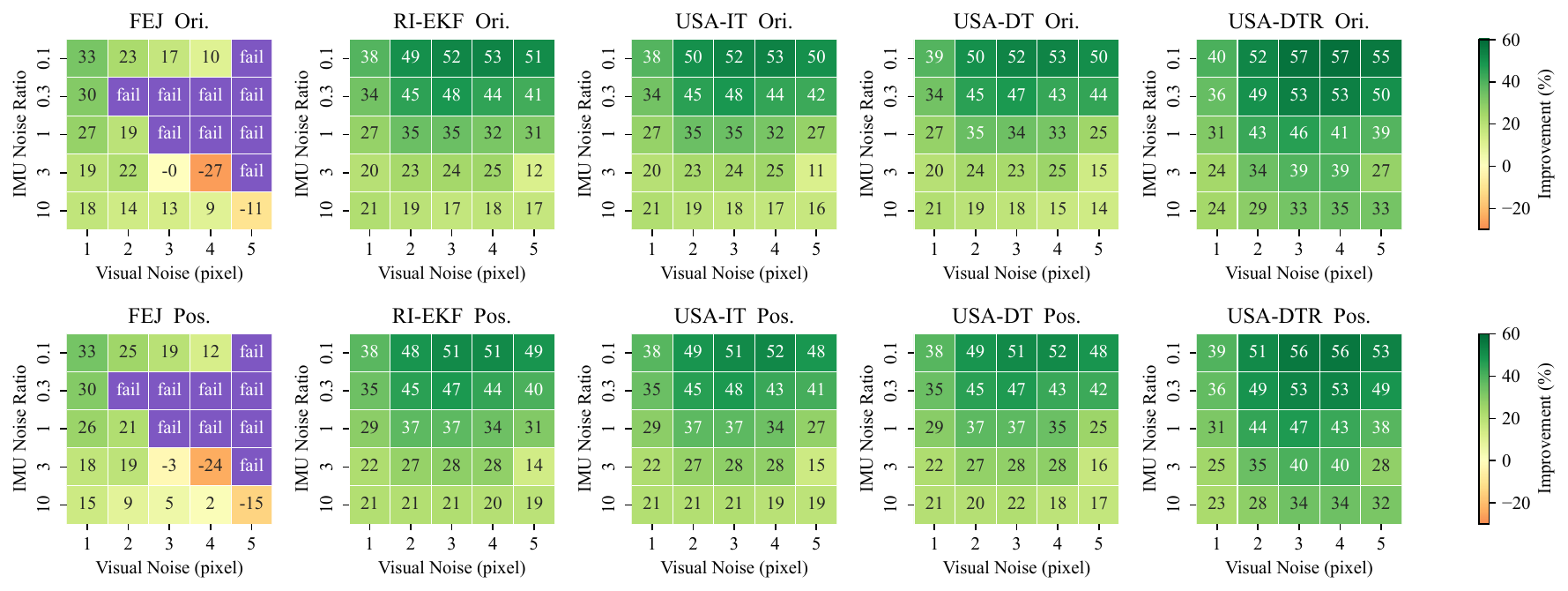}
    \caption{Average orientation and position RMSE improvements of the estimators compared to Std under different noise levels.}
    \label{fig:RMSE_noise}
\end{figure*}

\textbf{Accuracy:}
We test the six estimators described above in the MSCKF, SLAM, and Hybrid modes.
For each mode, each estimator, and each dataset, we run 1000 Monte Carlo simulations and report the average orientation and position RMSE in Table \ref{tab:RMSE_MSH3}.

\begin{table}[!ht] 
    \centering
    \caption{Average RMSE (deg/meter) of 1000 Monte Carlo simulations}
    \label{tab:RMSE_MSH3}
    \begin{threeparttable}          
            \setlength{\tabcolsep}{4.5pt}
            \begin{tabular}[]{ccccc}
                    \toprule
                    Dataset & Estimator & MSCKF & SLAM  & Hybrid \\
                    \midrule
                    \multirow{6}{*}{\rotatebox{90}{Udel-gore}} & Std & 0.886 / 0.284 & 0.950 / 0.220 & 0.907 / 0.213 \\
                     & FEJ & \textbf{0.840} / 0.303 & 0.560 / 0.187 & 0.553 / 0.183 \\
                     & RI-EKF & \textbf{0.840} / \textbf{0.281} & 0.499 / \textbf{0.174} & 0.490 / \textbf{0.172} \\
                     & \textbf{USA-IT} & \textbf{0.840} / \textbf{0.281} & 0.499 / \textbf{0.174} & 0.490 / \textbf{0.172} \\
                    & \textbf{USA-DT} & {\textbf{0.840} / \textbf{0.281}} & 0.499 / \textbf{0.174} & 0.491 / \textbf{0.172} \\
                    & \textbf{USA-DTR} & {\textbf{0.840} / \textbf{0.281}} & \textbf{0.497} / \textbf{0.174} & \textbf{0.489} / \textbf{0.172} \\
                    \midrule
                    \multirow{6}{*} {\rotatebox{90}{Udel-neig}} & Std & 4.093 / 26.337 & 4.047 / 27.109 & 3.714 / 24.409 \\
                    & FEJ & 3.441 / 22.709 & - & 3.334 / 21.796 \\
                    & RI-EKF & 3.211 / 20.391 & 2.412 / 16.111 & 2.403 / 15.559 \\
                    & \textbf{USA-IT} &  {3.208 / 20.378} & 2.412 / 16.109 & 2.395 / 15.517 \\
                    & \textbf{USA-DT} & \textbf{3.204} / \textbf{20.361} & 2.421 / 16.156 & 2.399 / 15.540 \\
                    & \textbf{USA-DTR}&  \textbf{3.204} / \textbf{20.361} & \textbf{2.110} / \textbf{14.604} &  \textbf{2.083}  / \textbf{13.965}  \\
                    \bottomrule
            \end{tabular}
    \end{threeparttable}       
\end{table}
During this test, USA-DTR consistently outperforms all other methods across all scenarios, which is expected, as it aligns the unobservable subspace with the current estimate while also reducing the linearization error within the information matrix. The RI-EKF and USA-IT methods demonstrate nearly identical performance. This similarity is anticipated, since, as discussed in Remark \ref{remark:ri_tekf}, both approaches share an equivalent linearized error-state system. Interestingly, the USA-DT method also exhibits performance comparable to RI-EKF and USA-IT, despite its greater simplicity. Unlike RI-EKF and USA-IT, USA-DT provides a straightforward closed-form solution and does not require the design of invariant error-state representations or auxiliary matrices. 
The RMSE value of USA-DT and USA-DTR is the same in MSCKF mode, because feature initialization is not performed in this mode.  
Regarding FEJ, it generally outperforms Std, except in the  SLAM mode on the Udel-neig dataset, where FEJ occasionally fails due to poor feature initialization in some runs.

To further compare the performance of the estimators under different noise levels, IMU noises are set to [0.1, 0.3, 1, 3, 10] times the default noise level, while the visual noise is set to [1, 2, 3, 4, 5] pixels.  For each noise level, we run 100 Monte Carlo simulations on the Udel-neig dataset with the Hybrid mode. It is obvious that the performance of the estimators degrades as the noise increases. Thus, to better visualize the results, we present the improvement compared to Std within the corresponding noise level:
\begin{equation}
    \text{Improvement} = \frac{\text{RMSE}_{\text{Std}} - \text{RMSE}_{\text{Other}}}{\text{RMSE}_{\text{Std}}}.
    \label{equ:improvement}
\end{equation}
The improvement is shown in Fig. \ref{fig:RMSE_noise}.
FEJ uses the predicted IMU states and the initial estimated feature positions as linearization points to ensure consistency. However, as the noise level increases, the accuracy of the first estimate deteriorates significantly. As a result, the Jacobian of FEJ deviates from the optimal linearization point, leading to a significant decline in accuracy, sometimes even worse than Std.
In contrast, the others evaluate the Jacobians using the current best estimates, thereby maintaining their Jacobian optimality. 
Among these estimators, USA-DTR consistently outperforms the others across all noise levels.

\textbf{Computational efficiency:}
We also demonstrate that USA can ensure consistency while not significantly increasing the computational burden.  To compare the computational efficiency of these estimators, we set the maximum number of SLAM features to [20, 40,  80] in the Udel-gore dataset with the  SLAM mode and the hybrid mode. The results of the MSCKF mode are omitted, as it does not maintain any SLAM features in the state vector, and the computational burdens of different estimators in this mode are almost the same.

\begin{figure}[htbp]
    \includegraphics[width=0.485\textwidth]{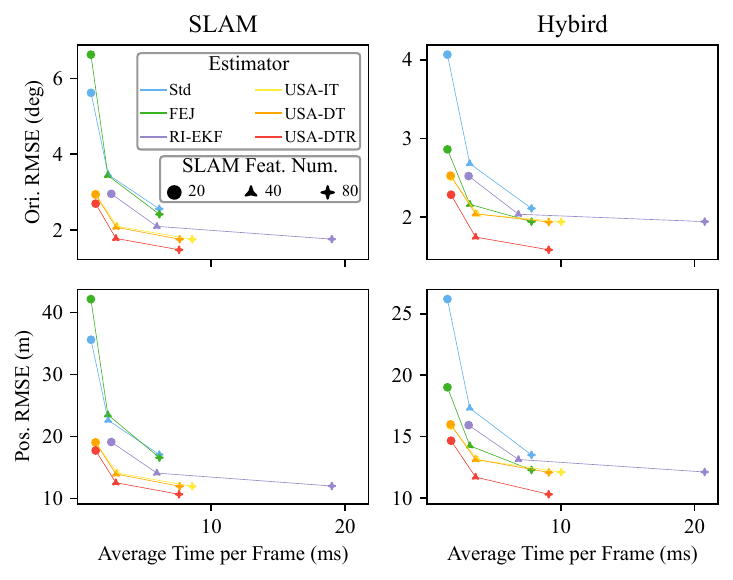}
    \caption{Average RMSE and time cost of processing one frame with different numbers of SLAM features. Points closer to the lower-left corner indicate better performance.}
    \label{fig:time_accuracy}
\end{figure}

Fig. \ref{fig:time_accuracy} reports the average orientation and position RMSE of the estimators together with the average time cost of processing one frame. The results indicate that the introduction of USA does not significantly increase the computational demand. Notably, USA-DTR exhibits the best performance among the three USA methods, demonstrating that the re-evaluation can enhance the accuracy without adding computational cost.
In contrast, RI-EKF incurs high computational costs when the number of landmarks is large due to a computational bottleneck in covariance propagation. 
To address this bottleneck, \cite{tian2025teskf} proposed an acceleration technique; however, this approach would increase the implementation complexity of the estimator.
 
\textbf{Consistency:}
To evaluate the consistency of the estimators, the NEES results within the Hybrid mode are presented in Fig. \ref{fig:udel_neighborhood_nees_compare}. The upper two subfigures display the average NEES of the orientation and position over time, respectively. Among these estimators, USA-DTR achieves the best NEES values, even better than RI-EKF. The lower two subfigures show the NEES frequency distribution. 
Theoretically, NEES follows a chi‑square distribution with 3 degrees of freedom, both for position and orientation. 
The results show that the NEES frequency histogram of USA-DTR  best fits the chi-square distribution, while RI-EKF, USA-IT, and USA-DT are slightly skewed to the right. 
\begin{figure}[htbp]
    \centering
    \includegraphics[width=0.485\textwidth]{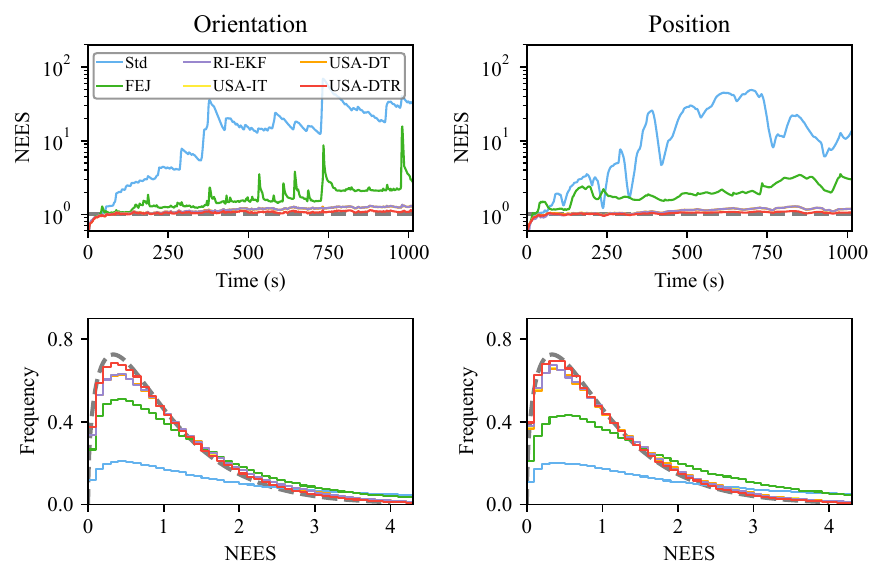}
    \caption{NEES results in Udel-neig in the hybrid mode. The upper two subfigures show the average NEES of the orientation and position over time, respectively. The lower two subfigures show the NEES frequency distribution, with the x-axis representing the NEES value and the y-axis indicating the frequency of occurrence. The lines of RI-EKF, USA-DT, and USA-IT overlap due to their similar performance.}
    \label{fig:udel_neighborhood_nees_compare}
\end{figure}
\begin{figure}[ht]
    \centering
    \includegraphics[width=0.485\textwidth]{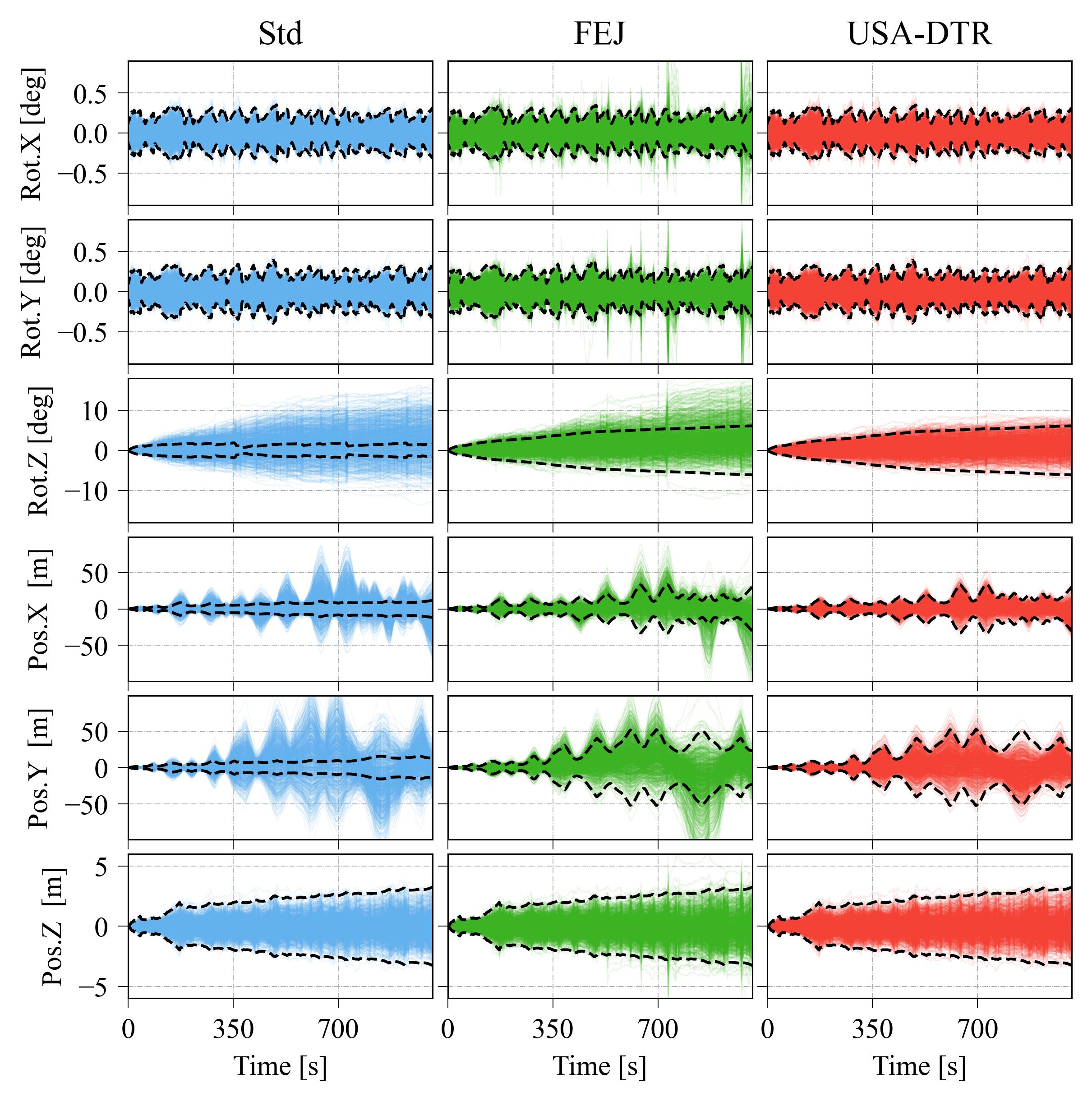}
    \caption{Orientation and position errors of 1000 runs in Udel-neig in the hybrid mode. The black dashed lines represent the 3-$\sigma$ bound. 
    The lines of RI-EKF, USA-IT, and USA-DT are similar to USA-DTR and are therefore omitted to save space.
    }
    \label{fig:udel_neighborhood_compare}
\end{figure}

Additionally, Fig. \ref{fig:udel_neighborhood_compare} shows the orientation and position errors together with the corresponding 3-$\sigma$ bounds. Note that here we report the orientation error using the globally-defined orientation error \cite{eskf} instead of the local error, because the last element in the global error corresponds exactly to the orientation around the gravity direction, i.e., $\bfN_\text{g}$.
Since  $\bfN_\text{g}$ falsely becomes observable in Std, the 3-$\sigma$ bound of the orientation error around the $z$-axis does not increase.
As a result, the errors of Std exceed the 3-$\sigma$ bound, indicating that the estimator is inconsistent. FEJ performs better than Std and has a similar bound to those of other estimators. However, its error exceeds the 3-$\sigma$ bound in some runs, mainly resulting from the bad feature initialization. 
The errors of RI-EKF, USA-IT, USA-DT, and USA-DTR are almost always within the bounds of 3-$\sigma$.

\textbf{Extension to local feature representation:}
The proposed method can also be utilized in the local feature representation (e.g., anchored inverse depth \cite{civera2008inverse}). 
The unobservable subspace with a local feature
representation has a similar form to \eqref{equ:unobservable_subspace}, where the only difference is that the block corresponding to features is constant. Since this block does not depend on states, the delayed feature initialization step will not cause misalignment. Thus, there is no need to align the unobservable subspace after the feature initialization step.
\begin{table}[!t] 
    \centering
    \caption{Average RMSE (deg/meter)  of 1000 runs within the local feature representation}
    \begin{threeparttable}          
            \setlength{\tabcolsep}{5pt}
            \begin{tabular}[]{cccccc}
                    \toprule
                    Dataset & Estimator &  MSCKF &  SLAM  & Hybrid \\
                    \midrule
                    \multirow{4}{*}{\rotatebox{90}{Udel-gore}} & Std & 0.886 / 0.284 & 0.507 / \textbf{0.171} & 0.497 / 0.169 \\
                     & FEJ & 0.840 / 0.303 & 0.497 / 0.173 & 0.485 / 0.170 \\
                     & \textbf{USA-IT} & \textbf{0.840} / \textbf{0.281} & \textbf{0.497} / \textbf{0.171} & \textbf{0.485} / \textbf{0.168} \\
                     &\textbf{USA-DT} & \textbf{0.840} / \textbf{0.281} & \textbf{0.497} / \textbf{0.171} & \textbf{0.485} / \textbf{0.168} \\
                    \midrule
                    \multirow{4}{*} {\rotatebox{90}{Udel-neig}} & Std & 4.093 / 26.337 & 3.169 / 21.150 & 3.026 / 19.967 \\
                                                                & FEJ & 3.441 / 22.708 & 2.760 / 18.414 & 2.579 / 16.982 \\
                                                                & \textbf{USA-IT} & {3.208} / {20.378} & {2.274} / {15.405} & \textbf{2.339} / \textbf{15.276} \\
                                                                & \textbf{USA-DT} & \textbf{3.204} / \textbf{20.361} & \textbf{2.271} / \textbf{15.389} & \textbf{2.339} / {15.278} \\
                    \bottomrule
            \end{tabular}
    \end{threeparttable}       
    \label{tab:RMSE_MSH4}
\end{table}

Table \ref{tab:RMSE_MSH4} shows the average orientation and position RMSE in the local feature representation.
The results demonstrate that the proposed methods, USA-IT and USA-DT, are also effective in the local feature representation. 
Compared to the results within the global representation in  Table \ref{tab:RMSE_MSH3}, the performance gap between Std and the other methods is smaller; 
we speculate that the unobservable subspace in the local feature representation contains a larger constant component, which results in a weaker inconsistency phenomenon compared to the global feature representation.

\subsection {Real-world Dataset Evaluation}
We further evaluate the proposed method through real-world experiments.
\begin{figure}
    \centering
    \includegraphics[width=0.99\linewidth]{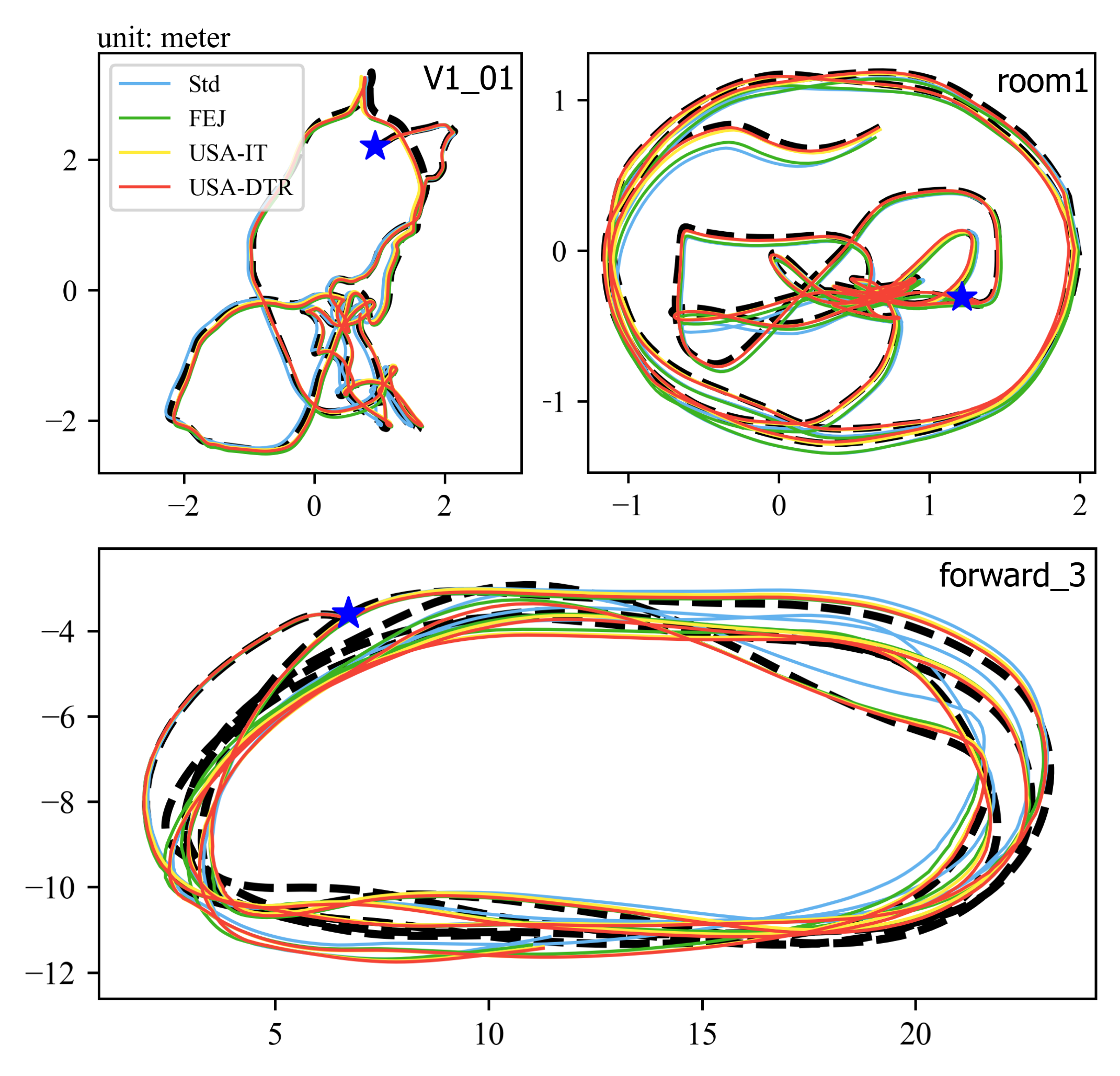}
    \caption{The groundtruth trajectory (black dashed line) and the estimated trajectories (colored solid lines) are aligned using the initial frame (blue star) as the reference point. To avoid overcrowding of the lines, only the trajectories during the first 50 or 100 seconds are shown.}
    \label{fig:single_posyawalign}
\end{figure}

\textbf{Indoor Datasets:}
We compare the proposed method with other methods on three public indoor datasets: EuRoC-MAV \cite{burri2016EuRoCMicroAerial}, TUM-VI \cite{schubert2018tum}, and UZH-FPV \cite{delmerico2019are}. 
During this test,  the feature numbers for SLAM and MSCKF are limited to 50 and 40, respectively. To enhance performance, camera intrinsics and time offset calibrations are enabled. For a fair comparison, all other parameters are kept at their default values in OpenVINS.

To assess the accuracy of the estimators, the estimated and ground truth trajectories are aligned based on the initial frame, as illustrated in Figure \ref{fig:single_posyawalign}. Table \ref{tab:pose_estim_error} reports the RMSE results on the three datasets. (USA-DT and RI-EKF are omitted as they are similar to USA-IT.) Due to the real-world complexities, such as incorrect data associations and noise models, the dataset results may not be as ideal as those from simulations. Nonetheless, compared to Std, the proposed USA still shows a substantial improvement in accuracy.

\begin{table}[htbp]

\caption{RMSE (deg/meter) with stereo configuration on Datasets EuRoC-MAV, TUM-VI, and UZH-FPV.}
\label{tab:pose_estim_error}    
\centering
\setlength{\tabcolsep}{3pt}
\begin{tabular}{lcccc}
\toprule
Dataset & Std & FEJ & {USA-IT} & {USA-DTR} \\
\midrule
V1\_01          & 1.295 / 0.101    & 1.220 / 0.076    & \textbf{0.598} / \textbf{0.074}    & 0.809 / 0.096 \\
V1\_02          & 2.951 / 0.126    & \textbf{1.317} / 0.127    & 1.686 / \textbf{0.102}    & 1.879 / 0.104 \\
V1\_03          & 2.119 / 0.224    & 2.273 / 0.206    & \textbf{1.649} / 0.214    & 2.136 / \textbf{0.197} \\
V2\_01          & 3.065 / 0.224    & 0.865 / 0.213    & 0.835 / \textbf{0.157}    & \textbf{0.759} / 0.193 \\
V2\_02          & 2.035 / 0.089    & 1.958 / 0.080    & \textbf{1.482} / \textbf{0.074}    & 2.401 / 0.099 \\
V2\_03          & 2.088 / 0.361    & 2.953 / 0.286    & \textbf{1.257} / \textbf{0.259}    & 1.294 / 0.293 \\
MH\_01          & 6.129 / 0.595    & 1.533 / 0.473    & 2.272 / 0.544    & \textbf{1.423} / \textbf{0.364} \\
MH\_02          & 4.454 / 0.295    & 1.097 / 0.506    & \textbf{0.820} / 0.210    & 1.622 / \textbf{0.199} \\
MH\_03          & \textbf{1.354} / \textbf{0.198}    & 2.299 / 0.295    & 2.288 / 0.313    & 1.804 / 0.260 \\
MH\_04          & 2.458 / 0.815    & 0.801 / \textbf{0.300}    & \textbf{0.627} / 0.393    & 0.988 / 0.409 \\
MH\_05          & 1.988 / 0.657    & 1.180 / 0.569    & 1.282 / \textbf{0.426}    & \textbf{1.059} / 0.476 \\
\midrule
room1     & 2.264 / 0.146    & 1.809 / 0.125    & 2.065 / \textbf{0.108}    & \textbf{1.381} / 0.108 \\
room2     & 14.95 / 0.334    & 3.032 / \textbf{0.200}    & 1.172 / 0.209    & \textbf{0.842} / 0.223 \\
room3     & 11.53 / 0.249    & \textbf{1.319} / 0.158    & 1.973 / 0.153    & 1.699 / \textbf{0.145} \\
room4     & 3.930 / \textbf{0.211}    & 3.049 / 0.242    & \textbf{2.186} / 0.217    & 2.257 / 0.229 \\
room5     & 2.458 / 0.137    & 3.208 / \textbf{0.119}    & 2.710 / 0.189    & \textbf{1.875} / 0.133 \\
room6     & 7.106 / 0.302    & 4.205 / \textbf{0.148}    & \textbf{4.047} / 0.288    & 4.361 / 0.249 \\
\midrule
forward\_3     & 7.796 / 0.995    & \textbf{5.046} / \textbf{0.653}    & 5.310 / 0.684    & 5.152 / 0.667 \\
forward\_5     & 1.631 / 0.456    & \textbf{0.860} / 0.401    & 1.121 / \textbf{0.392}    & 1.051 / 0.629 \\
forward\_6     & 3.471 / 0.439    & \textbf{2.652} / \textbf{0.325}    & 2.689 / 0.371    & 2.869 / 0.377 \\
forward\_7     & 5.876 / 1.090    & \textbf{4.142} / 1.036    & 4.710 / \textbf{0.875}    & 4.347 / 0.994 \\
forward\_9     & 0.984 / \textbf{0.443}    & 1.027 / 0.597    & 1.004 / 0.505    & \textbf{0.941} / 0.516 \\
forward\_10    & 1.267 / 0.379    & 0.881 / 0.387    & 0.850 / 0.359    & \textbf{0.843} / \textbf{0.319} \\
45\_2     & 1.369 / 0.512    & 1.015 / \textbf{0.468}    & \textbf{0.905} / 0.537    & 0.932 / 0.555 \\
45\_4     & 0.917 / \textbf{0.473}    & \textbf{0.605} / 0.505    & 0.647 / 0.565    & 0.720 / 0.488 \\
45\_12     & 1.174 / \textbf{0.384}    & 1.177 / 0.508    & 0.745 / 0.398    & \textbf{0.721} / 0.399 \\
45\_13     & 1.977 / 0.493    & \textbf{1.168} / 0.484    & 1.342 / \textbf{0.427}    & 1.293 / 0.521 \\
45\_14     & 2.417 / \textbf{0.374}    & 2.412 / 0.452    & \textbf{2.315} / 0.417    & 2.416 / 0.440 \\
\midrule
Average                  & 3.484 / 0.383    & 1.900 / 0.343    & 1.744 / \textbf{0.326}    & \textbf{1.720} / 0.334 \\

\bottomrule
\end{tabular}
\end{table}

\begin{table*}[!htb] 
    \centering
    \caption{RMSE (deg/meter) on Dataset KAIST-Urban}
    \begin{threeparttable}          
            \setlength{\tabcolsep}{5pt}
            \begin{tabular}[]{cccccccc}
                    \toprule
                      & \multirow{2}{*} {Estimator} & urban28   & urban32 &urban34 &urban38 &urban39 & \multirow{2}{*} {Average}  \\
                      & & {\scriptsize (11212m)} &\scriptsize (7131m) &\scriptsize (7881m) &\scriptsize (11192m) &\scriptsize (10678m) \\
                    \midrule
                    \multirow{4}{*}{\rotatebox{0}{Mono}} & Std & \textbf{1.736} / \textbf{60.964} & 3.606 / 70.477 & 1.977 / 41.034 & 4.928 / 59.569 & 4.253 / 51.849 & 3.300 / 56.779 \\
                     & FEJ & - / -  & 2.500 / 122.49 & 2.574 / 169.57 & - / -  & 4.156 / 132.55 & 3.077 / 141.54 \\
                     & \textbf{USA-IT}  & 1.877 / 63.774 & 2.161 / 44.256 & \textbf{1.516} / 29.301 & \textbf{1.396} / \textbf{17.747} & \textbf{1.510} / \textbf{25.560} & \textbf{1.692} / \textbf{36.128} \\
                     &\textbf{USA-DT}  & 1.771 / 65.067 & \textbf{1.862} / \textbf{40.128} & 1.639 / \textbf{27.005} & 1.614 / 23.880 & 1.692 / 33.084 & 1.716 / 37.833 \\
                    \midrule
                    \multirow{4}{*} {\rotatebox{0}{Stereo}} &Std & 3.596 / 40.042 & 1.391 / 21.855 & 2.134 / 40.659 & 2.386 / 23.046 & 2.030 / 21.892 & 2.308 / 29.499 \\
                    & FEJ  & 1.668 / \textbf{10.738} & 1.523 / 27.540 & \textbf{1.585} / \textbf{25.154} & 1.651 / 14.609 & 2.378 / 16.025 & 1.761 / 18.813 \\
                    & \textbf{USA-IT} & \textbf{1.553} / 12.418 & \textbf{1.222} / \textbf{12.741} & 1.691 / 27.120 & \textbf{1.457} / \textbf{11.938} & \textbf{1.505} / \textbf{12.084} & \textbf{1.485} / \textbf{15.260}\\
                    & \textbf{USA-DT}  & 1.788 / 15.379 & 1.337 / 18.920 & 1.851 / 30.637 & 1.563 / 14.009 & 1.585 / 12.802 & 1.625 / 18.349\\
                    \bottomrule
            \end{tabular}
    \end{threeparttable}       
    \label{tab:RMSE_Kaist}
\end{table*}

Figure \ref{fig:exp_yaw_nees} shows the NEES for the orientation around gravity (Yaw) over time, which is used to evaluate the consistency of the estimators. Because the Std incorrectly treats Yaw as an observable, its estimation exhibits overconfidence. In contrast, the other estimators ensure the correct observability, so their NEES values are significantly lower than those of the Std method, indicating improved consistency.

\begin{figure}
    \centering
    \includegraphics[width=0.99\linewidth]{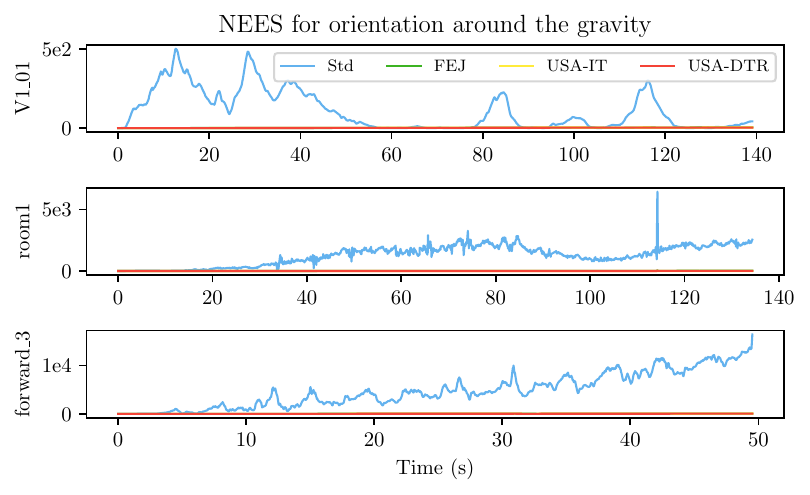}
    \caption{The NEES value over time on subset V1\_01, room1, and forward\_3. 
    The NEES of Std is much greater than that of the others.
    The lines of FEJ, USA-IT, and USA-DTR overlap.}
    \label{fig:exp_yaw_nees}
\end{figure}

\textbf{Outdoor Datasets:}
The KIAST-Urban dataset \cite{jeong2019complex} is an outdoor dataset that captures the complex characteristics of urban environments. The dataset is collected using a vehicle equipped with multiple sensors, including an IMU operating at 200 Hz and a stereo camera system with a baseline of 0.47 meters. 
The lengths of the vehicle's trajectories are about 10 km, with the maximum speed exceeding 20 m/s, as shown in Fig. \ref{fig:traj_urban39}.
To better parameterize landmarks with large depths, we utilize the local inverse depth feature representation. Furthermore, dynamic objects such as vehicles and pedestrians are removed from the images using YOLOv11-seg \cite{khanam2024yolov11}  segmentation during our experiments.

\begin{figure}[htbp]
    \centering
    \includegraphics[width=0.99\linewidth]{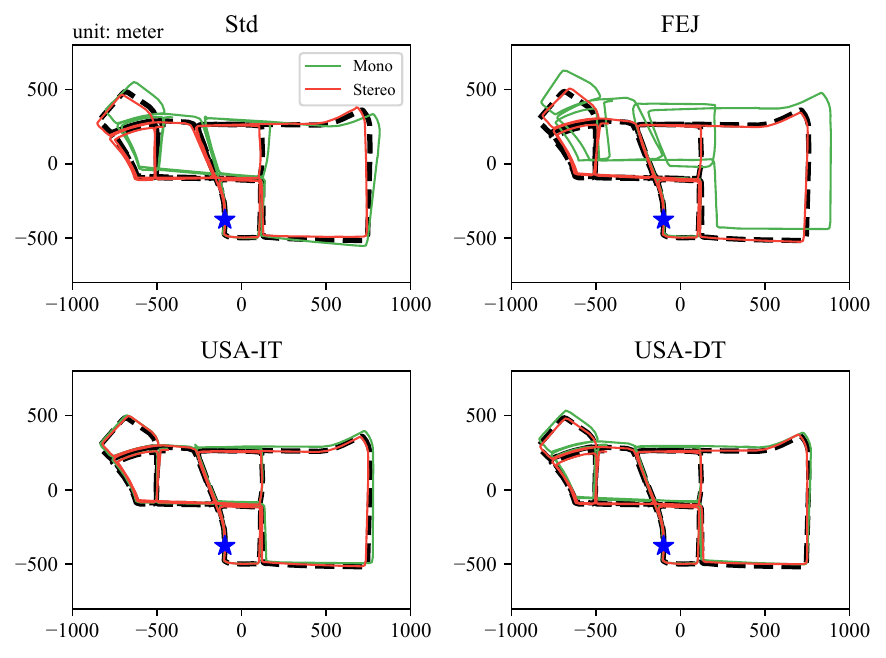}
    \caption{Ground truth (black dash) and estimated trajectories (colored solids) with monocular and stereo configurations on the urban39 subset. The starting point is highlighted using the blue star. Compared to FEJ, USA-IT and USA-DT perform more stably across different configurations. }
    \label{fig:traj_urban39} 
\end{figure}

Table \ref{tab:RMSE_Kaist} reports the RMSE results of Std, FEJ, USA-IT, and USA-DT  on the KAIST-Urban dataset. 
Among the estimators, USA-IT and USA-DT exhibit the best accuracy. 
More importantly, compared to FEJ, both USA-DT and USA-IT consistently demonstrate stable performance in both monocular and stereo configurations.
As mentioned in Section \ref{sec:sim_usa}, to ensure consistency, FEJ sacrifices the optimality of the Jacobian. When the initialization of features is poor, such as under monocular conditions, the performance of FEJ is affected by the non-optimality of the Jacobian and can even be worse than that of Std. 
In the stereo configuration, FEJ improves estimation accuracy compared to Std, benefiting from the improved feature initialization provided by the large stereo baseline.

\section{Conclusion}
This paper investigates the inconsistency issue in VINS by introducing two novel methods: Unobservable Subspace Evolution (USE) and Unobservable Subspace Alignment (USA). Through the USE analysis framework, we systematically analyzed how the unobservable subspace evolves across the full estimation pipeline and revealed that observability misalignment, introduced by specific estimation steps such as correction and delayed feature initialization, serves as the antecedent of observability mismatch. This insight provides a more complete understanding of the dynamics underlying inconsistency and clarifies a long-standing misconception regarding the consistency properties of MSCKF-based estimators.
Building upon this analysis, we propose USA as a simple yet effective solution paradigm. USA eliminates inconsistency by selectively intervening only at the specific estimation steps causing misalignment, while preserving Jacobian optimality and avoiding the need for complex state redefinitions or structural reformulations. Two realizations of USA are designed: the transformation-based USA offers closed-form implementations with the flexibility to balance accuracy and efficiency, while the re-evaluation-based USA can be combined with transformations to enhance accuracy without added computation. Extensive simulations and real-world experiments demonstrate that USA achieves superior or comparable performance in both accuracy and consistency compared to existing state-of-the-art methods, while maintaining higher computational efficiency and flexibility.

The proposed methods open several promising directions for future research. In particular, extending USE and USA to resolve inconsistencies in other estimation problems, such as cooperative localization, map-based SLAM, and LiDAR–inertial odometry, may provide principled and efficient means for addressing inconsistency across a wider range of robotic systems.

\appendix

\subsection{Transformation-based USA within Covariance Form}
\label{appendix:trans_cov}
Transforming the information matrix described by \eqref{equ:La_align} can be implemented within the covariance form:
\begin{equation}
    \bfP^{\smalloplus} \gets \bfT^{-1} \bfP^{\smalloplus} {\bfT^{-1}}^\top
\end{equation}
where $\bfP^{\smalloplus}$ is the covariance matrix. 

For the direct transformation, its inverse also has a simple form:
\begin{equation}
   \bfT^{-1} = \bfI - \frac{1}{1+\boldsymbol{\beta}^\top\boldsymbol{\alpha}} \boldsymbol{\alpha}  \boldsymbol{\beta}^\top \label{equ:USA2_Tinv}
\end{equation}
which can be efficiently implemented by optimizing the order of matrix multiplications:
\begin{equation}
    \bfT^{-1} \bfP^{\smalloplus} {\bfT^{-1}}^\top =  \bfP^{\smalloplus} + \bfA +\bfA^\top + \mathbf{B}
\end{equation}
with 
\begin{align}
    \bfA &= \boldsymbol{\alpha}'(  \boldsymbol{\beta}^\top \bfP^{\smalloplus}) \\
    \mathbf{B}&= \boldsymbol{\alpha}' \left((\boldsymbol{\beta}^\top \bfP^{\smalloplus})\boldsymbol{\beta} \right) \boldsymbol{\alpha}'^\top\\
    \boldsymbol{\alpha}' &= -\frac{1}{1+\boldsymbol{\beta}^\top\boldsymbol{\alpha}}\boldsymbol{\alpha}
\end{align}
We can see that the matrix multiplications between large matrices are avoided in the above implementation.

For the indirect transformations, $\bfT^{-1}$ also has a simple form:
\begin{align}
     \bfT^{-1} = \mathbb{T}(\hat{\bfx}^{\smalloplus})^{-1} \mathbb{T}(\hat{\bfx}^{\smallominus}) \label{equ:USA1_Tinv} 
\end{align}
which can be decomposed as the multiplication of several simple matrices:
\begin{align}
    {\bfT^{-1}} & = \bfT_{\mathcal{F}} \bfT_{\mathcal{W}} \bfT_{\mathcal{I}}  \bfT_{R} 
\end{align}
where
\begin{align}
    \bfT_{R}  &= \begin{bmatrix}
        {{^G}\hat{\bfR}^{\smalloplus}_{I_\tau}}^{-1}{{^G}\hat{\bfR}^{\smallominus}_{I_\tau}}&\\
        & \bfI_{12+6n+3m}\end{bmatrix} \\
    \bfT_{\mathcal{I}} &= \begin{bmatrix}
        \bfI_3 &\\
        [{^G}\hat{\bfp}_{I_\tau}^{\smallominus} - {^G}\hat{\bfp}_{I_\tau}^{\smalloplus}]_\times{{^G}\hat{\bfR}^{\smalloplus}_{I_\tau}} & \bfI_3 &\\
        [{^G}\hat{\bfv}_{I_\tau}^{\smallominus} -  {^G}\hat{\bfv}_{I_\tau}^{\smalloplus}]_\times{{^G}\hat{\bfR}^{\smalloplus}_{I_\tau}} && \bfI_3 &\\
       &&&\bfI_{6+6n+3m}\\
    \end{bmatrix} \\
    \bfT_{\mathcal{W}} &= \begin{bmatrix}
        \bfI_{15} &\\
        &\mathbb{W}(\hat{\bfx}^{\smalloplus})^{-1} \mathbb{W}(\hat{\bfx}^{\smallominus})\\
        && \bfI_{3m}
    \end{bmatrix} \\
    \bfT_{\mathcal{F}} &= \begin{bmatrix}
        \bfI_{3} \\
         & \bfI_{12 + 6n} \\
        \left( \mathbb{F}(\hat{\bfx}^{\smallominus}) - \mathbb{F}(\hat{\bfx}^{\smalloplus}) \right) {{^G}\hat{\bfR}^{\smalloplus}_{I_\tau}}& & \bfI_{3m}\\
    \end{bmatrix}
\end{align}
$\bfT_{\mathcal{I}} $, $  \bfT_{\mathcal{W}} $, and $ \bfT_{\mathcal{F}} $ can be continuously decomposed into the product of $2$, $2n$, and $m$ elementary matrices with a block size of $3\times 3$, respectively. Thus, the indirect transformation can be regarded as a sequence of simple elementary multiplications, which can be efficiently implemented.

\subsection{Status Evolution Triggered by State Augmentation and Marginalization} 
Since state augmentation and marginalization do not correct the state estimates, they will not cause observability misalignment, i.e.,
\begin{subequations}
    \begin{align}
        & \delta(\text{Aligned}, \text{Augmentation}) =\text{Aligned},\\ 
        & \delta(\text{Aligned}, \text{Marginalization}) =\text{Aligned}.
    \end{align}
\end{subequations}
\label{app:imu_aug}%
The proofs are as follows:

\textbf{Augmentation:}
Every time an image is received, the IMU state is predicted to the current time. Then the system state is augmented with a clone of the current IMU pose $\hat{\boldsymbol{\pi}}_{\tau}$. 
During IMU state augmentation, the estimate and the information matrix are updated as:
\begin{equation}
    \hat{\bfx}^{\smalloplus} =\begin{bmatrix}
        \hat\bfx^{\smallominus} \\ \hat{\boldsymbol{\pi}}_{\tau}
    \end{bmatrix}, \quad \bfLa^{\smalloplus}  = \begin{bmatrix}
        \bfLa^{\smallominus} & \bfZo \\ \bfZo & \bfZo 
    \end{bmatrix} + \bfH_{\text{aug}}^\top \bfR^{-1} \bfH_{\text{aug}},
\end{equation}
where 
\begin{equation}
    \bfH_{\text{aug}} = \begin{bmatrix}
        \bfI_6&\bfZo_{6\times (9+3m+6n) } &-\bfI_6
    \end{bmatrix}.
\end{equation}
and $\bfR $ is a positive definite matrix that infinitely approaches zero, ensuring that the augmentation operation in the information form is equivalent to that in the covariance form. Since the linearization point does not change during augmentation, we have:
 \begin{equation}
    \mathbf{N}^{\smalloplus} = \bfN(\hat\bfx^{\smalloplus}) = \begin{bmatrix}
        \bfN(\hat\bfx^{{\smallominus}}) \\
        \hline 
        \begin{matrix}
            \bfZo_{3\times 3}& -{{^G}{\hat\bfR}_{I_\tau}^{\smallominus}}^{{-1}} \bfg \\
            \bfI_3 & [{^G}{\hat\bfp}_{I_\tau}^{\smallominus}]_\times \bfg\\
        \end{matrix}
    \end{bmatrix}.
    \label{equ:85}
\end{equation}

\textbf{Marginalization:}
If the number of the cloned IMU poses exceeds a threshold or the features are no longer tracked, the oldest clone or the features are marginalized out from the state vector. 
Let the estimate, information matrix, and unobservable subspace before marginalization be partitioned as:
\begin{equation}
    \hat\bfx^{\smallominus}  = \begin{bmatrix}
        \hat \bfx_{\text{r}} \\  \hat \bfx_{\text{m}}
    \end{bmatrix},  \bfLa^{\smallominus}   = \begin{bmatrix}
        \bfLa_{\text{rr}} & \bfLa_{\text{rm}} \\
        \bfLa_{\text{mr}} & \bfLa_{\text{mm}}
    \end{bmatrix},\mathbf{N}^{\smallominus} = \begin{bmatrix}
        \bfN(\hat \bfx_{\text{r}}) \\
        \bfN(\hat \bfx_{\text{m}})
    \end{bmatrix}.
\end{equation}
where $ \hat \bfx_{\text{m}}$ is the state variable that is going to be marginalized out, and $\hat \bfx_\text{r}$ is the rest of $\hat\bfx^{\smallominus} $. During state marginalization, the estimate and the information matrix are updated as:
\begin{equation}
    \hat{\bfx}^{\smalloplus} =\hat\bfx_{\text{r}}, \quad \bfLa^{\smalloplus}  = \bfLa_{\text{rr}} - \bfLa_{\text{rm}} \bfLa_{\text{mm}}^{-1} \bfLa_{\text{mr}}.
\end{equation}
Then, we have 
\begin{subequations}
    \begin{align}
    \mathbf{N}^{\smalloplus} & = \texttt{null}({\bfLa}^{\smalloplus})\\
    & = \bfN(\hat{\bfx}_{\text{r}}) \\
    & = \bfN(\hat{\bfx}^{\smalloplus}). \label{equ:88c}
\end{align}
\end{subequations}
Equations \eqref{equ:85} and \eqref{equ:88c} demonstrate that after state augmentation and marginalization, the unobservable subspace remains Aligned.

\bibliographystyle{IEEEtran}
\bibliography{ref.bib,ref2.bib}

\end{document}